\documentclass[10pt,journal,compsoc]{IEEEtran}
\usepackage{natbib}
\setcitestyle{square,sort,comma,numbers}
\usepackage{multirow, booktabs}
\usepackage{siunitx}
\usepackage{makecell}

\usepackage{mathrsfs}

\usepackage{graphicx}

\usepackage[font=small,labelfont=bf]{caption}
\usepackage{amsthm}

\usepackage{algorithmic}
\usepackage{algorithm}

\usepackage{amsmath}
\usepackage{epstopdf} 
\usepackage{subcaption}
\usepackage{lipsum}
\usepackage{amssymb,amsmath,amsthm}

\newtheorem{lemma}{Lemma}

\newtheorem{problem}{Problem}
\newtheorem{definition}{Definition}

\usepackage{mathtools} 

\usepackage{hyperref}
\hypersetup{
	colorlinks=true,
	linkcolor=blue,
	filecolor=blue,      
	urlcolor=blue,
	citecolor=blue
}

\begin{document}
	\newcommand*{\everymodeprime}{\ensuremath{\prime}}
	%
	\title{SoulMate: Short-text author linking through Multi-aspect temporal-textual embedding}
	%
	%
	%
	%
	
	\author{\{Saeed~Najafipour, 
		Saeid~Hosseini\}*,  
		Wen~Hua,
		Mohammad~Reza~Kangavari, and
		Xiaofang~Zhou

		\IEEEcompsocitemizethanks{
			\IEEEcompsocthanksitem * Saeed Najafipour and Saeid Hosseini are co-first authors
			\IEEEcompsocthanksitem S. Najafipour, S. Hosseini (Affiliated) and M. R. Kangavari are with computational cognitive model research lab., School of computer engineering, Iran University of Science and Technology, Iran. \protect\\Email: {saeed\_najafi}@comp.iust.ac.ir, {kanagvari}@iust.ac.ir, {saeid.hosseini}@uq.net.au	
			\IEEEcompsocthanksitem S. Hosseini is with Cyber Security Lab., Singapore University of Technology and Design, Singapore.
			\IEEEcompsocthanksitem W. Hua and Xiaofang Zhou are with the School of Information Technology and Electrical Engineering, University of Queensland, Brisbane, Australia.\protect\\E-mail: {w.hua}@uq.edu.au, {zxf}@itee.uq.edu.au
		}
		\thanks{}}
	
	%
	%
	
	\markboth{}%
	{Shell \MakeLowercase{\textit{et al.}}: Bare Demo of IEEEtran.cls for Computer Society Journals}
	%
	
	
	
	\IEEEtitleabstractindextext{%
		\begin{abstract}			
			Linking authors of short-text contents has important usages in many applications, including Named Entity Recognition (NER) and human community detection. However, certain challenges lie ahead. Firstly, the input short-text contents are noisy, ambiguous, and do not follow the grammatical rules. Secondly, traditional text mining methods fail to effectively extract concepts through words and phrases. Thirdly, the textual contents are temporally skewed, which can affect the semantic understanding by multiple time facets. Finally, using the complementary knowledge-bases makes the results biased to the content of the external database and deviates the understanding and interpretation away from the real nature of the given short text corpus. To overcome these challenges, we devise a neural network-based temporal-textual framework that generates the tightly connected author subgraphs from microblog short-text contents. Our approach, on the one hand, computes the relevance score (edge weight) between the authors through considering a portmanteau of contents and concepts, and on the other hand, employs a stack-wise graph cutting algorithm to extract the communities of the related authors. Experimental results show that compared to other knowledge-centered competitors, our multi-aspect vector space model can achieve a higher performance in linking short-text authors. Additionally, given the author linking task, the more comprehensive the dataset is, the higher the significance of the extracted concepts will be.
		\end{abstract}
		\vspace{-3mm}
		}		
		\maketitle
		
		\IEEEpeerreviewmaketitle
		\IEEEraisesectionheading{\section{Introduction}\label{introductionSection}}
		\IEEEPARstart{G}{enerating} the subgraphs with similar vertices finds important applications in numerous domains: (i) In recommender systems \cite{chen2018attention}\cite{livne2014citesight}, the members in a subgraph can enrich the history of other cold-start members \cite{cao2017embedding}\cite{manotumruksa2018contextual}\cite{cao2018attentive}. (ii) In community detection, the subgraphs can identify groups of correlated users \cite{cai2017community}\cite{belesiotis2018spatio}. (iii) In propagation networks \cite{cai2017community}\cite{hosseini2018mining},  the group-based immunization policies \cite{hosseini2018exploiting} \cite{Zhang2016} can better control the burst of contagions (i.e. gossips). Nowadays, the social networks record the commonly brief textual contents of the authors that are generated in a high-throughput rate. Given a graph $G$ of short-text authors and the query author $n_q$, our aim is to find a subgraph $\tilde{g}_q$ comprising of highly similar authors to $n_q$. The NP-hard subgraph mining problem can be initiated by the computation of edge weights between authors and completed by a stack-wise graph-cutting algorithm. As the main step in obtaining of subgraphs, several approaches \cite{ganguly2016author2vec}\cite{Rosen-Zvi2012} have been proposed to compute the similarity weight among authors. However, the nature of short-text contents causes certain obstacles. Such \textit{challenges} are instantiated as follows:\\
\noindent\textit{\textbf{Challenge 1 (Mismatched Author Contents)}}\\
		\label{Challenge1}
		\indent Short-text contents are typically informal and include abbreviations, misspellings, and possible errors. For instance, ``afternoon'' is informally used as ``arvo''. Similarly, ``Brisbane'' (in Australia) is usually abbreviated to ``BNE'' and called as ``Brissie''. As a result, current text mining approaches (e.g., topic modeling \cite{nguyen2015improving}\cite{cao2017you} and other heuristics \cite{Hosseini2014}\cite{hosseini2017leveraging}) may not gain sufficient statistical signals and \textit{mismatch} the textual contents of the similar authors. Consequently, the correlation edge weight between the pair of authors will be calculated incorrectly.\\		
		\textit{\textbf{Challenge 2 (Context Temporal Alignments)}}\\
		\indent Vector representation models analyze each word in the context of others.  GloVe \cite{Pennington2014} consumes word pair co-occurrences and CBOW \cite{mikolov2013efficient} predicts a word given the surrounding context. However, the state-of-the-art models \cite{dumais2004latent}\cite{Mikolov2013a} ignore the reality that the word proximity patterns alter in various temporal facets. To witness the fact, we set up an observation on our Twitter dataset \cite{Hosseini2014}.\\
		\vspace{-7mm}
		\begin{figure}[H]
			\centering
			\begin{subfigure}{0.49\linewidth}
				\centering
				\includegraphics[width=1\textwidth]{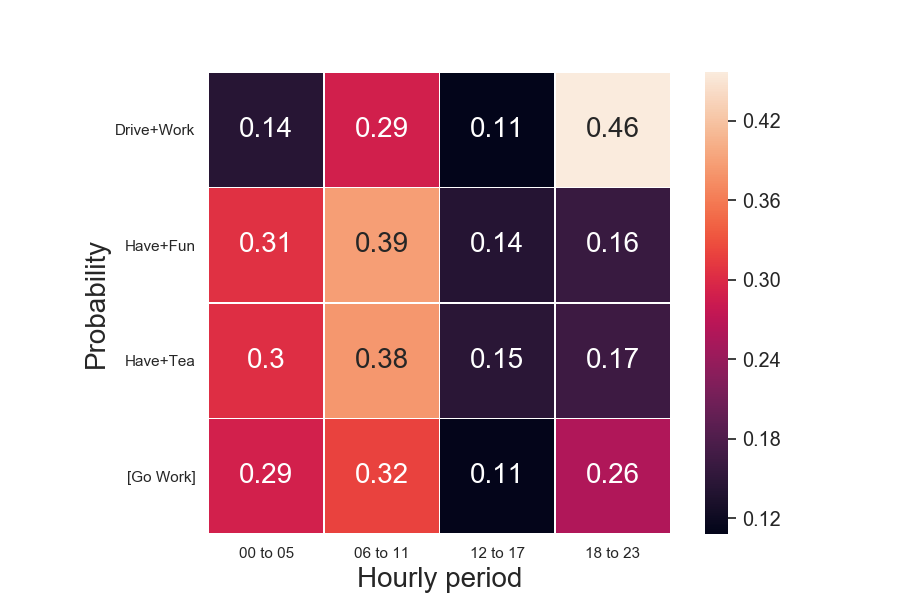}
				\vspace{-5mm}
				\caption{Hour dimension}
				\label{fig:HeatHourDimension}
			\end{subfigure}
			\centering
			\begin{subfigure}{0.49\linewidth}
				\centering
				\includegraphics[width=1\textwidth]{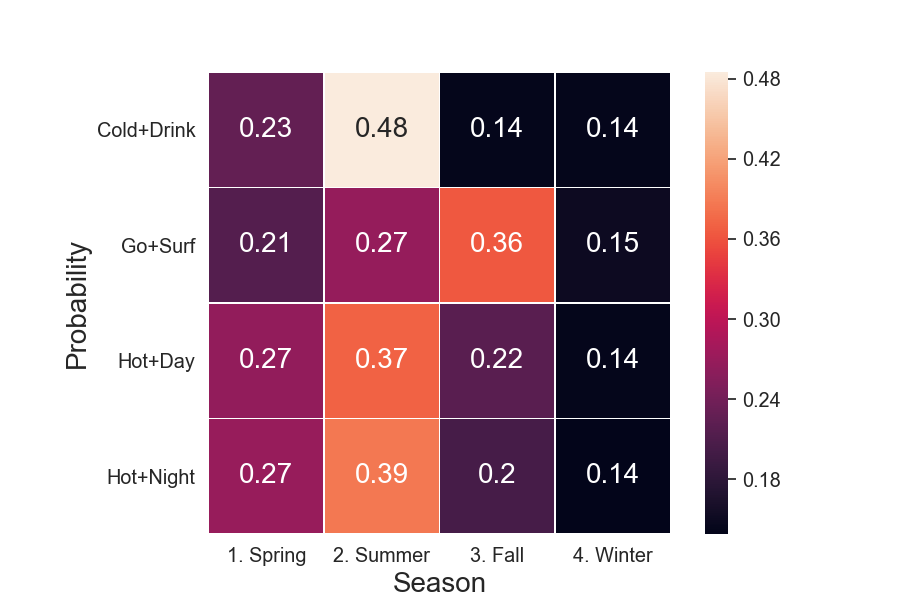}
				\vspace{-5mm}
				\caption{Season dimension}
				\label{fig:HeatSeasonDimension}
			\end{subfigure}
			\hfill
			\vspace{-2mm}
			\caption{Co-occurrence probability}
			\label{fig:Heatmap}
			\vspace{-3mm}
		\end{figure}
			\begin{table*}
		\centering
		\def\arraystretch{1.5}
		\tiny
		\begin{tabular}{|l|l|l|}
			\hline
			\textbf{Concept}                                              & \textbf{Tweet \#1}                                                                                                                                                                    & \textbf{Tweet \#2}                                                                                                                                                                         
			\\ \hline
			\begin{tabular}[c]{@{}l@{}}overconsumption\end{tabular}    & \begin{tabular}[c]{@{}l@{}}I've drunk so much tea today!\end{tabular}   & \begin{tabular}[c]{@{}l@{}}they told me to mince two cabbages today. cabbage nightmares\end{tabular} \\ \hline
			resistance                                                    & \begin{tabular}[c]{@{}l@{}}Everything against me, I'm still here\end{tabular} & \begin{tabular}[c]{@{}l@{}}@H, You feel that burn? muscles working, you becoming fit.\end{tabular}	
			\\ \hline
			dullness                                                          & No ones up. Boring.                                                                                                                                                                   & \begin{tabular}[c]{@{}l@{}}Can't help it. Depression just hits me in the face\end{tabular}                                                                                             \\ 
			
			\hline
			\begin{tabular}[c]{@{}l@{}}wish\end{tabular} & \begin{tabular}[c]{@{}l@{}}Would love to join @m but I'm exhausted!\end{tabular}                                                                    & \begin{tabular}[c]{@{}l@{}}@a so sad!, sending you hugs thoughts and love x\end{tabular}                                                                          \\ 				
			\hline
		\end{tabular}
		\vspace{-2mm}
		\caption{Conceptual relevance}
		\label{tab:AuthorRelevance}
		\vspace{-7mm}
	\end{table*}

		\noindent As Fig. \ref{fig:Heatmap} demonstrates, the distribution probabilities for word pairs can differ in various temporal dimensions. While most people talk about going to work between 6 to 11 am, half others drive to work in the evening (Fig. \ref{fig:HeatHourDimension}). People mostly tweet about \emph{Cold+Drink}, \emph{Hot+Day}, and \emph{Hot+Night} during the summer and such word pairs proximate far less during the winter (Fig. \ref{fig:HeatSeasonDimension}).\\
		\noindent\textit{\textbf{Challenge 3 (Ignoring Conceptual Relevance)}} \\
		\indent Recent works in finding author similarities \cite{alinani2018aggregating}\cite{alshareef2018recommending}\cite{li2018conference}\cite{li2019personalised}, compute the textual relevance between authors through exact or approximate matching. Though the Skip-gram \cite{Mikolov2013a}\cite{Pennington2014} enriches the contents using word vectors to optimize the similarity weights between authors. Using external knowledge bases (e.g. WordNet and Wikipedia) can relax the negative effects of the noisy contents \cite{Hua2017}. However, after all textual enrichments, while the contents belonging to a pair of authors completely or partially turn up irrelevant, we need to how they share some common concepts. As instantiated in Table \ref{tab:AuthorRelevance}, where the tweets (\#1 and \#2) belong to distinct authors and the \emph{contents} are mainly irrelevant, they may still reflect conceptual relevance (e.g. \emph{overconsumption} in row 1).
		Hence, a new approach must confront the above obstacles.\\
		\noindent\textbf{Contributions.} While \emph{our} previous work \cite{Hua2017} detects the concepts of a single tweet with the help of an external knowledge base (KB), the proposed framework in this paper identifies concepts through unsupervised clustering in the tweet space itself, thus eliminating the bias and deviation caused by a KB. Furthermore, the concept distribution of all the tweets is aggregated for each author, which is then utilized as an important feature for determining author similarity. We also enrich the short-text inference models with semantic vector space approaches to further understand the contextual (contents+concepts) relationships between short-text authors. Moreover, inspired by \emph{our} previous work \cite{hosseini2017leveraging}, we devise a temporal-textual embedding model to track the temporal alignments of the short text contents. 
Our contributions are fourfold:\\		
		\vspace{-6mm}
		\begin{itemize}
			\item To the best of our knowledge, we develop the first neural network-based temporal-textual embedding approach that can enrich short-text contents using word vectors that are collectively learned through multiple temporal dimensions.			
			\item We analyze two clustering methods of DB-Scan and K-medoids in extracting various concepts from a corpus. 
			\item We design a stack-wise graph-cutting algorithm to extract the maximum spanning trees (subgraphs) from the authors weighted graph.
			\item We propose a temporal-textual framework that can
			achieve better effectiveness in generating of the
			highly relevant short-text authors.
		\end{itemize}
		\vspace{-1mm}
		\indent The rest of our paper is as follows: in Sec. \ref{relatedwork}, we study the literature; in Sec \ref{problemstatement}, we provide the problem and our framework; in Sec. \ref{methodology} and \ref{experiments} we respectively explain our model and experiments. The paper is concluded in Sec. \ref{conclusion}.
		\vspace{-9mm}
		\section{Related Work}
		\label{relatedwork}
		\vspace{-1mm}
		\noindent As briefed in Table \ref{relatedwork}, the related work comprises word embedding, user similarity, and semantic understanding.				
		\vspace{-5mm}
		\subsection{Word Embedding}
		\label{Sec_Word_Embedding_related}		
		\vspace{-1mm}
		\textit{Word embedding} \cite{rumelhart1986learning} accurately associates each word in documents with a meaningful continuous vector. The applications are multidisciplinary and involve Information Retrieval (IR), Natural Language Processing (NLP) \cite{ling2015two} \cite{hu2017opinion} \cite{majumder2017deep}, and Recommendation Systems (RS) \cite{park2018conceptvector} \cite{cao2017embedding} \cite{grbovic2015commerce} \cite{fu2018novel}. 
		Traditional LSI \cite{Deerwester1990} approach captures the relevance between terms and concepts through \textit{Singular Value Decomposition}(SVD), while \textit{Bag-of-Words}(BOW) \cite{Manning2008} disregards the order of words and considers the word frequencies. Expectation-Maximization(EM) can equip the BOW model \cite{ling2015not} \cite{talley2011database} to reduce ambiguity in machine translation. The continuous BOW (CBOW) \cite{mikolov2013efficient} ensures that the words with similar meanings will be constituted by close embeddings in the latent space. 
		The \textit{Word2vec} \cite{Mikolov2013a} extracts meaningful syntactical and semantical regularities from word pairs and is classified into CBOW and skip-gram. While the Global Vector model (GloVe) \cite{Pennington2014} consumes word pair co-occurrences to accomplish word embedding, the CBOW model \cite{mikolov2013efficient} predicts a word given the surrounding context. Based on our Twitter dataset, the CBOW model surpasses the GloVe approach in the standard analogy test (Section \ref{Comparing-vector-representation-models}).
		 Vector representation has other types: Paragraph2Vec \cite{le2014distributed}, ConceptVector \cite{park2018conceptvector}, Category2Vec \cite{zhu2018exploiting}, Prod2Vec \cite{grbovic2015commerce}. Moreover, \cite{nguyen2015improving} includes \textit{topic models} to collectively generate a word from either Dirichlet multinominal or the embedding module. \cite{zhu2018exploiting} enriches the embedding with Knowledge Graphs to eliminate ambiguity and improve similarity measures. However, the state-of-the-art models \cite{dumais2004latent}\cite{Mikolov2013a} ignore the fact that the word proximity patterns alter in different temporal facets. Even temporal models \cite{bamler2017dynamic} \cite{dubossarsky2017outta} rely on a single temporal aspect and pretermit semantical relations \cite{rosin2017learning}. But our embedding model can employ an infinite number of temporal facets.		
		\vspace{-3mm}
		\begin{table}[H]
			\label{tab:literature}
			\centering
			\caption{Literature}
			\vspace{-3mm}
			\def\arraystretch{1.5}
			\tiny
			\begin{tabular}{cll}
				\Xhline{2\arrayrulewidth}
				Category                                                                          & Approachs                                                           & Refferences                                                                                                                                                                                                                                                  \\ \hline
				\multirow{4}{*}{\begin{tabular}[c]{@{}c@{}}Word\\ Embedding\end{tabular}}         & Embedding Only                                        & \cite{ling2015not}\cite{mikolov2013efficient}\cite{mikolov2013linguistic}\cite{Pennington2014}       \\
				& Graph Analysis                                                      & \cite{cao2017you}\cite{li2019personalised}\cite{seyler2018information}\cite{zhu2018exploiting}                                                                                       \\
				& \begin{tabular}[c]{@{}l@{}}Matrix Factorization (MF)\end{tabular} & \cite{alshareef2018recommending}\cite{chen2018attention}\cite{koren2009matrix}\cite{lee1999learning}\cite{li2019personalised}\cite{Mikolov2013a} \\
				& Temporal                                                            & \cite{bamler2017dynamic}\cite{dubossarsky2017outta}\cite{li2017neural}\cite{manotumruksa2018contextual}                                                                             \\ \hline
				\multirow{5}{*}{\begin{tabular}[c]{@{}c@{}}User\\ Similarity\end{tabular}}        & \begin{tabular}[c]{@{}l@{}}Collaborative Filtering\end{tabular}   & \cite{cao2018attentive}\cite{chen2010short}\cite{li2019personalised}\cite{lian2015content}\cite{seyler2018information}                                            \\
				& Neural Networks                                                     & \cite{grover2016node2vec}\cite{liu2018multi}\cite{Mikolov2013a}\cite{perozzi2014deepwalk}\cite{wu2017mllda}                                                        \\
				& Graph Analysis                                                      & \cite{cao2017you}\cite{ganguly2016author2vec}\cite{li2019personalised}\cite{seyler2018information}\cite{yan2012tweet}                                             \\
				& Author-Oriented                                                     & \cite{alinani2018aggregating}\cite{alshareef2018recommending}\cite{li2018conference}\cite{li2019personalised}                                                                         \\
				& Temporal                                                            & \cite{alharbi2016learning}\cite{alinani2018aggregating}\cite{fani2017temporally}\cite{li2017neural}\cite{manotumruksa2018contextual}                               \\ \hline
				\multirow{4}{*}{\begin{tabular}[c]{@{}c@{}}Semantic\\ Understanding\end{tabular}} & Neural Networks                                                     & \cite{graves2012supervised}\cite{le2014distributed}\cite{mikolov2013efficient}\cite{shen2014learning}                             \\
				& Expansion                                                      & \cite{buckley1995automatic}\cite{efron2012improving}\cite{tang2017end}\cite{zhai2001model}                                                                                          \\
				& Topic Modeling                                                      & \cite{alshareef2018recommending}\cite{blei2009topic}\cite{li2016topic}\cite{li2018conference}                                                                                        \\
				& Concept-Oriented                                                    & \cite{Hua2017}\cite{kim2013context}\cite{song2011short}                                                                                                                                             \\ \Xhline{2\arrayrulewidth}
			\end{tabular}
			\vspace{-5mm}
		\end{table}
		\vspace{-5mm}
		\subsection{User Similarity}
		\vspace{-1mm}
		The similarity between users can be computed by their contextual information through a similarity function (e.g. Cosine or Pearson). Subsequently, the collaborative filtering \cite{chen2010short}, graph-theoretic model \cite{yan2012tweet}, or other classification methods \cite{diaz2012real} \cite{uysal2011user} can be used to group pertinent users. Embedding models like User2Vec \cite{liu2018multi} and Mllda \cite{wu2017mllda} utilize associated vectors to find the order of user relations. Other embedding models like node2vec \cite{grover2016node2vec} and DeepWalk \cite{perozzi2014deepwalk} are motivated by the skip-gram approach and treat each user as a node in social networks. Author2Vec \cite{ganguly2016author2vec} combines the content and link information of the users to better predict the true connection between individuals.
		\cite{alinani2018aggregating}\cite{alshareef2018recommending}\cite{li2018conference}\cite{li2019personalised} compute the user similarity through exact and approximate textual matching. In contrast, \cite{Mikolov2013a} maps users to a latent space. Unlike \cite{fani2017temporally} which proposes a distributional representation for user temporal contents, majority of other embedding models neglect the time factor. Hence, we integrate multi-facet time-based clusters into text embedding to infer both the temporal and textual correlations between users.
		\vspace{-5mm}
		\subsection{Semantic Understanding}
		\vspace{-1mm}
		Textual semantics can be learned through various approaches. Early NER methods \cite{McCallum2003} \cite{Zhou2001} employ classification techniques to label the entities in a document. However, NER models cannot function effectively on noisy short-text contents. Topic models \cite{li2016topic}\cite{blei2009topic} exploit latent topics through word distributions. However, since they do not effectively retrieve the statistical cues from short-text contents, the semantic labeling task is left out abortive.
		Recent state-of-the-art models like CBOW \cite{mikolov2013efficient} and ParagraphVEC \cite{le2014distributed} are proved to be beneficial to the understanding of the textual contents. Expansion models \cite{tang2017end} are inspired by query expansion techniques \cite{buckley1995automatic} \cite{efron2012improving} \cite{zhai2001model} and enrich initial textual contents by complementary relevant contents. More recent Deep Neural Network models such as CNN \cite{shen2014learning} and RNN \cite{graves2012supervised} facilitate short-text understanding through classification. The labeling and feature extraction modules can further promote human language inference. Some other works \cite{song2011short} \cite{kim2013context} \cite{Hua2017} exploit global concepts from the corpus to better represent the semantics in each document. More specifically, the derived concepts from microblog brief contents can increase noise tolerance and improve short text semantic understanding. Hence, our proposed framework employs clustering modules besides embedding procedure to simultaneously take advantage of concepts and contents.
		\vspace{-7mm}
		\begin{figure}[htp]
			\centering
			\includegraphics[width=0.9\linewidth]{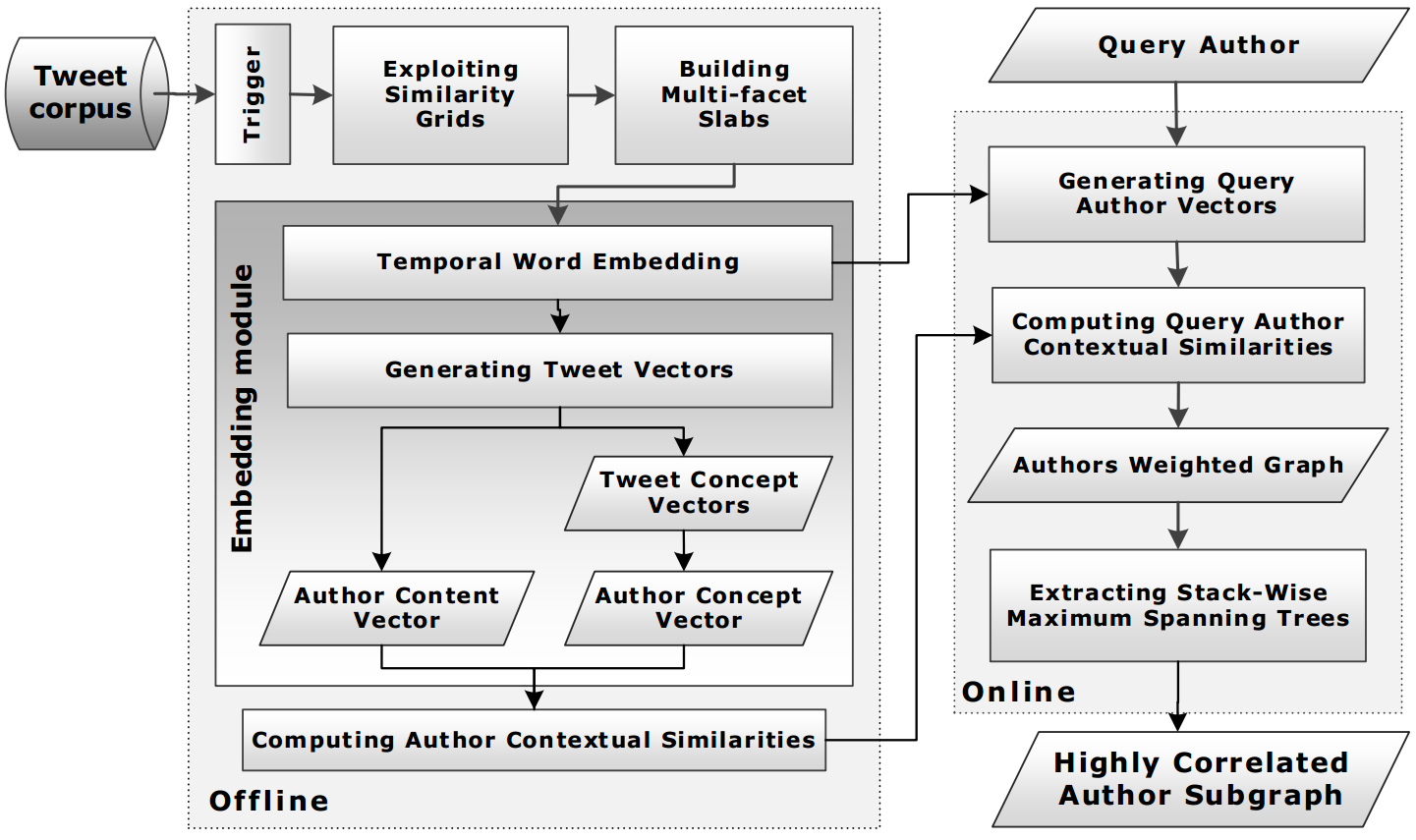}
			\vspace{-2mm}
			\caption{Framework}
			\label{fig:Framework}
			\vspace{-6mm}
		\end{figure}
	\vspace{-4mm}
		\section{Problem Statement}
		\label{problemstatement}		
		\vspace{-1mm}
		In this section we elucidate preliminary concepts, problem statement for author linking, and our proposed framework.
		\vspace{-9mm}
		\subsection{Preliminary Concepts}
		\vspace{-1mm}
		We commence with definitions as formalized below:
		\vspace{-2mm}
		\begin{definition} (node) Every Node, denoted by $n_i$ $\in$ $\mathbb{N}$ represents a distinguished author of short-text contents in a social network.
		\end{definition}	
		\vspace{-3mm}
		
		\begin{definition} (short-text message)
			\label{def:shorttext_message}			
			$m_j \in \mathbb{M}$ advocates a short-text that is composed and published by an author $n_i$. Each message has an identity ($m_j$), the associated author (i.e. $n_i$), and the time-stamp ($m_j.t$). Accordingly, $\mathbb{M}$ ($\mathbb{M}=\{M_1, M_2,\dots,M_n\}$) includes all short-text contents, where $M_i$ delineates the set of short-text messages that are owned by the author $n_i$.
		\end{definition}	
		\vspace{-2mm}

		\vspace{-2mm}
		\begin{definition} (latent temporal facets)
			The associated time-stamp can be interpreted by multiple time facets, where each facet represents a latent parameter $z^k$ in $\mathbb{T} = \lbrace
			z^1,z^2,...,z^t\rbrace$.
		\end{definition}	
		\vspace{-2mm}
		\begin{definition} (temporal slab)
			Each latent temporal dimension $z^x$ can comprise $\eta$ splits $z^x = \lbrace s_{1}^{x},s_{2}^{x},...,s_{\eta}^{x} \rbrace$, e.g. 7 splits for day dimension where each split advocates an individual day in the week. Accordingly, the \textit{Unifacet Temporal Slabs} are built via the merging of similar splits.
		\end{definition}	
		\vspace{-3mm}	
		\begin{definition} (vector representation)
			\vspace{-1mm}
			Given the corpus of textual contents $\mathcal{C}$, for elucidation purposes, the word embedding model conveys each vocabulary $v_i \in \mathbb{V}$ with a distinguished real-valued vector $\vec{v_i}$ which is obtained in conjunction with other word vectors.
		\end{definition}	
		\vspace{-4mm}
		\begin{definition} (authors weighted graph) $G=(\mathbb{N},\mathbb{L})$ is the authors' weighted graph which includes all the authors $\mathbb{N}$ with pertinent links $\mathbb{L}$ between them. The link $l_{ij} \in \mathbb{L}$ describes the similarity weight between two authors ($n_i$,$n_j$) and can be computed using various approaches.
		\end{definition}	
		\begin{definition} (query subgraph) The query subgraph denoted by $\tilde{g_q} \subset G$ contains the set of authors that are highly correlated with the input query user $n_q \in \mathbb{N}$.
		\end{definition}
		We not only use the vectorization for words, but also devise a textual time-aware model to represent the tweets, concepts, and authors by vectors.	
		\vspace{-5mm}	
		\subsection{Problem Definition}		
		\vspace{-1mm}		
		\begin{problem} (extracting hierarchical time-aware slabs)
			Given the set of authors $\mathbb{N}$, messages $\mathbb{M}$, and the temporal latent factors $\mathbb{T}$, we aim to extract the set of uni-facet temporal slabs through uniting similar splits in each facet. The child temporal facets are hierarchically affected by the parent dimension(s).
			\label{ExtractingHierarchical}
			\vspace{-2mm}
		\end{problem}		
		\noindent \emph{The weighted graph reflects the similarity between authors.}
		\vspace{-2mm}
		\begin{problem} (computing contextual similarities)	
			Given the set of temporal facets $\mathbb{T}$ and the messages of each author (e.g. $M_i$ for $n_i$), we aim to compute the correlation weight between each pair of authors ($l_{ij}\equiv Link(n_i,n_j) | \forall n_i,n_j \in \mathbb{N}\:\&\:l_{ij} \in \mathbb{L}$).
			\label{ComputingContextual}
			\vspace{-2mm}
		\end{problem}
		\begin{problem} (author linking)
			Given the author weighted graph $G$, we aim to mine the single query subgraph $\tilde{g}_q$ containing the set of authors, where all of them are highly correlated with every query author $n_q \in \mathbb{N}$. 
			\label{LinkingHighly}
			\vspace{-2mm}
		\end{problem}
		\vspace{-5mm}
		\subsection{Framework Overview}
		Intuitively, the problem of author linking (Prob. \ref{LinkingHighly}) can be divided into two steps: (1) to compute the similarity weights between authors (Prob. \ref{ComputingContextual}). (2) to employ a stack-wise graph cutting algorithm to optimize the number of exploited subgraphs and maximize the intra-subgraph correlations. Figure \ref{fig:Framework} illustrates our framework for linking authors of short-text contents through a multi-aspect temporal textual embedding model.
		In the offline part, we use the microblog contents to acquire the multi-facet grids that reflect the similarity weight between temporal splits (e.g. 24 hourly splits). Subsequently, we employ clustering models to build uni-facet time-based slabs by merging similar splits. We propose and investigate a temporal-textual embedding model to construct tweet vectors by comprising word vectors. Author content vectors are similarly formed through uniting vectorized tweets. We then cluster similar tweets to explore hidden concepts, where each tweet concept vector corresponds to a set of concepts. Inherently, the tweet concept vectors together form the author concept vector. Accordingly, each author can be collectively represented by the \textit{content} and \textit{concept} vectors (so-called \textit{context}), that can be further utilized to compute the similarity weight between any given pair of authors.
		In the online part, we aim to discover a subgraph that includes a set of highly correlated authors to the query author. We firstly generate the contextual vectors of the query author and update the authors' similarity matrix. Finally, we employ a simple but effective stack-wise graph cutting algorithm to extract the output subgraph - as maximum spanning trees.
		\vspace{-4mm}
		\begin{figure}[H]
			\centering
			\begin{subfigure}{0.49\linewidth}
				\centering
				\includegraphics[width=1\textwidth]{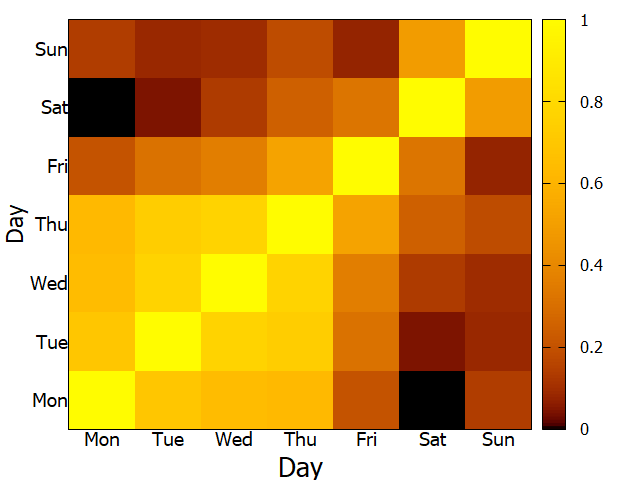}
				\caption{TF-IDF similarity matrix}
				\label{fig:daySimilarity}
			\end{subfigure}
			\centering
			\begin{subfigure}{0.25\linewidth}
				\centering
				\includegraphics[width=1\textwidth]{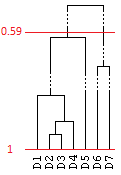}
				\caption{Day slabs}
				\label{fig:dayhac}
			\end{subfigure}
			\vspace{-3mm}					
			\hfill			
			\caption{The day latent facet}
			\label{fig:daydimension}
			\vspace{-5mm}
		\end{figure}
		\begin{table}[H]
		\renewcommand{\arraystretch}{1.3}		
		\centering
		\def\arraystretch{1.5}
		\tiny
		\begin{tabular}{c}
			\hline
			Day Slabs\\
			\hline
			\hline		
			$\lbrace$D1=Monday,D2=Tuesday,D3=Wednesday,D4=Thursday,D5=Friday$\rbrace$\\
			$\lbrace$D6=Saturday,D7=Sunday$\rbrace$\\
			\hline
		\end{tabular}
	\vspace{-2mm}
	\caption{Day Slabs, threshold set to 0.59}
	\label{tab:tabDay}
		\vspace{-6mm}
	\end{table}
		\section{Methodology}
		\label{methodology}
				\vspace{-1mm}
		\subsection{Offline Phase} 
		\label{offline_phase}
		\subsubsection{Constructing multi-facet dynamic slabs} 
		\label{Constructing-Similarity-Grids}
		\vspace{-1mm}
		For the similarity matrix, we measure the \textit{Textual} sameness between each pair of splits (e.g. Sunday and Monday in day dimension). To proceed, we congregate the textual contents of each temporal split. Accordingly, every temporal facet can be assigned with a vector where every cell will contain the short-text contents in the relevant split. We use a modified \textit{TF-IDF} algorithm (Eq. \ref{eq:TFIDF}) to find the weight of each word in the textual contents of every temporal facet:
		\vspace{-3mm}
		\begin{equation} 
		\small
		\label{eq:TFIDF}
		\mathbf{\hat{w}}(t_i,S_k^l)=\frac{f(t_i,S_k^l)}{Max_{(t\in S^l)}\{f(t,S_k^l)\}} \times Log \frac{N}{N(t_i)}
		\vspace{-2mm}
		\end{equation}
		Here, $N$ designates the total number of the splits and $N(t_i)$ is the number of splits at which the term $t_i$ has appeared. While $S_k^l$ is the textual contents of split $k$ of the latent facet $l$, $\frac{f(t_i,S_k^l)}{Max_{(t\in S^l)}\{f(t,S_k^l)\}}$ normalizes the term frequency. Correspondingly, every split can be signified by a vector $\vec{S}_k^l$ where the cells contain the weights of the terms. Eventually, a similarity measure (e.g. \textit{Cosine}) can report how correlated each pair of splits are. The number of dimensions,  (e.g. binary facets of $z^h$ and $z^d$ for the hour and day latent factors) can be decided based on the sparsity of the dataset as well as the complexity of the solution. 		
		Note that people behave differently at the same hour during various days, take weekdays and weekends into consideration. This shows how the child facet can be affected by its parent latent factor.  Hence, unlike our prior work \cite{hosseini2017leveraging}, we heed the effects of the parent(s) on the child temporal facets. For instance, $z_h \subset z_d$ elucidates that the hour dimension comes under the sway of the parent temporal dimension (i.e. $z_d$). 
		As implemented in \cite{hosseini2017leveraging}, the bottom-up \textit{Hierarchical Agglomerative Clustering} (\textit{HAC} via \textit{complete linkage}) can bundle similar temporal splits in each latent temporal facet to shape the final temporal slabs. The \textit{threshold} of the HAC model may place impertinent splits into the same cluster or incorrectly toss relevant splits to separate slabs.
		Figure \ref{fig:daydimension} depicts both the similarity grid and the clustering dendrogram for the day dimension. On the one hand, threshold 1.0 will place the everyday entity in a distinctive slab (no clustering). On the other hand, threshold 0.59 results in more meaningful slabs as reported in Table \ref{tab:tabDay} (weekday versus weekends). Since we consider the influence of the parent facets, for the hour dimension, we will need to consider two similarity matrices (one for each daily slabs). Figure \ref{fig:HourSimilarity} illustrates the hour similarity grids based on which the dendrograms are obtained as shown in Figure \ref{fig:HACs}. Note that as reported in Table \ref{tab:MFSSHour}, we will have two sets of clusters for the hour dimension, where each of them are devoted to the pertinent daily slab.
		\vspace{-4mm}
		\begin{figure}[H]
			\centering
			\begin{subfigure}{0.49\linewidth}
				\centering
				\includegraphics[width=1\textwidth]{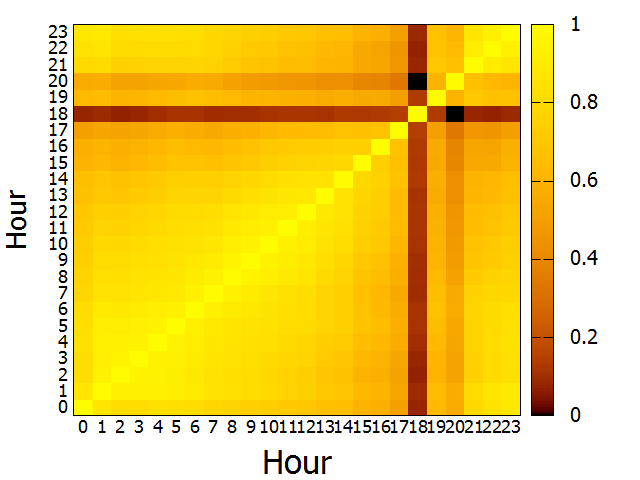}
				\vspace{-5mm}		
				\caption{\textit{weekdays}}
				\label{fig:mondayFriday}
			\end{subfigure}
			\centering
			\begin{subfigure}{0.49\linewidth}
				\centering
				\includegraphics[width=1\textwidth]{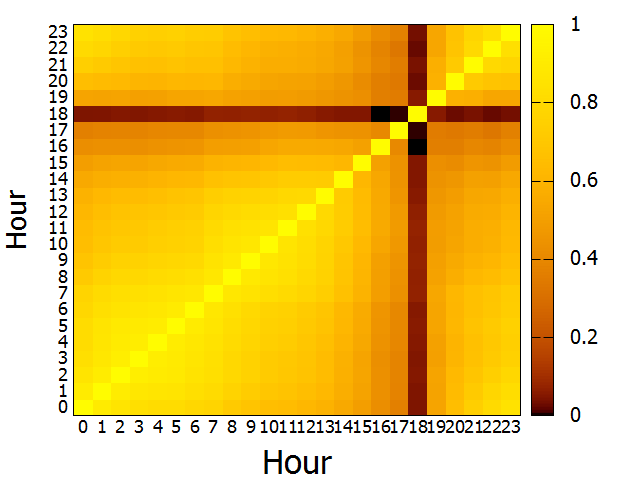}
				\vspace{-5mm}
				\caption{\textit{weekend}}
				\label{fig:saturdaySunday}
			\end{subfigure}
			\hfill
			\vspace{-3mm}
			\caption{Hour similarity grids based on daily temporal slabs}
			\label{fig:HourSimilarity}
			\vspace{-8mm}
		\end{figure}
		\begin{figure}[H]
			\centering
			\begin{subfigure}{0.49\linewidth}
				\centering
				\includegraphics[width=1\textwidth]{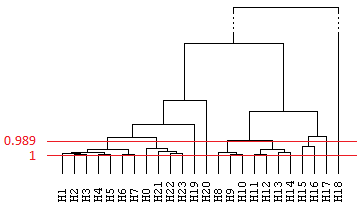}
				\vspace{-6mm}
				\caption{\textit{weekdays}}
				\label{fig:mfss0Hac}
			\end{subfigure}
			\centering
			\begin{subfigure}{0.49\linewidth}
				\centering
				\includegraphics[width=1\textwidth]{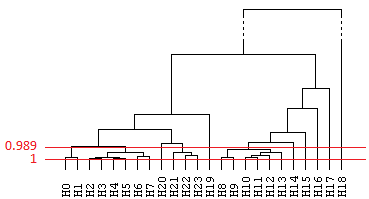}
				\vspace{-4mm}
				\caption{\textit{weekend}}
				\label{fig:mfss1HAC}
			\end{subfigure}
			\hfill
			\vspace{-3mm}
			\caption{hourly slabs hierarchical dendrograms}
			\label{fig:HACs}
			\vspace{-6mm}
		\end{figure}
		\begin{table}[H]
			\renewcommand{\arraystretch}{1.3}
			\centering
			\def\arraystretch{1.5}
			\tiny
			\begin{tabular}{c|c}
				\hline
				Monday till Friday (0.989)  &  Saturday and Sunday (0.989)\\
				\hline
				\hline
				$\lbrace$1,2,3,4,5,6,7$\rbrace$$\lbrace$8,9,10$\rbrace$ & $\lbrace$0,1$\rbrace$$\lbrace$2,3,4,5,6,7$\rbrace$  \\  $\lbrace$11,12,13,14$\rbrace$$\lbrace$15,16$\rbrace$$\lbrace$17$\rbrace$  &  $\lbrace$8,9,10,11,12,13$\rbrace$$\lbrace$14$\rbrace$$\lbrace$15$\rbrace$$\lbrace$16$\rbrace$ \\ $\lbrace$18$\rbrace$$\lbrace$19$\rbrace$$\lbrace$20$\rbrace$$\lbrace$21,22,23,0$\rbrace$        &       $\lbrace$17$\rbrace$$\lbrace$18$\rbrace$$\lbrace$19$\rbrace$$\lbrace$20$\rbrace$$\lbrace$21,22,23$\rbrace$\\
				\hline
			\end{tabular}
		\vspace{-2mm}
			\caption{Extracting hour slabs affected by the day slabs}
			\label{tab:MFSSHour}
			\vspace{-4mm}
		\end{table}
		\vspace{-4mm}
		\subsubsection{Word embedding models}
		\label{Sec:Word_Embedding}
		\vspace{-1mm}
		Informal short-text contents come with excessive noise and writing errors. Hence, as discussed in Challenge 1 (Section \ref{introductionSection}), recent text mining approaches including topic models \cite{nguyen2015improving}\cite{cao2017you} fail to obtain significant statistical cues to match the textual contents of the similar authors. On the other hand, the correlation weight between a pair of microblog authors ($u,v$) will be computed incorrectly when their respective exact textual contents ($O_u,O_v$) are considered. To cope with this challenge, the semantic vector space models \cite{sugawara2015context,Mikolov2013a,mikolov2013efficient,Pennington2014,Deerwester1990} retrieve the vector representation of each word. As the first solution to correctly compute the semantic relevance between authors, one can construct a decently ordered list of similar words to each comprising word  $v_i$ in an author's contents $o_u$, denoted by $\vec{v}_i^o$. Accordingly, the textual contents of each author will be represented by a new \textit{encyclopedic semantic representation} form ($O^\prime_u$) where every word $v_i \in O_u$ will be replaced by the top $\zeta$ most similar words from the associated vector $\vec{v}_i^o$. To this end, we can choose four embedding models: \textit{Singular Value Decomposition}(SVD) \cite{sugawara2015context}, Skip-gram \cite{Mikolov2013a}, CBOW \cite{mikolov2013efficient}, and also GloVe \cite{Pennington2014}. SVD computes the word vectors without training and using matrix operations over the co-occurrence matrix. For three models the well-trained vectors are iteratively enumerated through complex operations (e.g. forward and backward propagation). While the CBOW model estimates the center word of the window by the one-hot vector of the surrounding context (order is important), the skip-gram calculates the co-occurrence probability of the surrounding words with the middle word. Nevertheless, both models return the word vectors that are trained in the hidden layer. GloVe consumes the word co-occurrence matrix, where the model converges toward the optimized values in context and main vectors.
		\vspace{-5mm}
		\subsubsection{Temporal word embedding} 
		\label{Temporal-word-embedding}
		\vspace{-1mm}
		\noindent As elucidated in Section \ref{introductionSection}, the word proximity patterns change in various temporal facets. However, current word embedding models \cite{dumais2004latent}\cite{Mikolov2013a} ignore this reality. Also, notice that the CBOW algorithm can pass the word analogy test better than other vector space models (Sec. \ref{Sec_Word_Embedding_related}). Hence, we devise our novel time-aware embedding model based on CBOW, named as TCBOW, to better track the multi-aspect temporal-textual variations in short-text contents. Note that the temporal slabs monotonically capture the temporal alterations in textual contents through the clustering of similar splits in each temporal dimension. Accordingly, we argue that the time-aware embedding should gain an understanding of each of the slabs and subsequently predict unforeseen observations through merging the knowledge from all the slabs. Hence, we firstly devise a TCBOW module which functions on the slabs of all latent factors.\\		
		Figure \ref{fig:cbow} depicts the diagram for the slab-based TCBOW where $k$ is a single slab in dimension $l \in \mathbb{T}$. 
		The input layer contains the number of $C$ one-hot encoded input words $\{\textbf{x}^l_{k\{1\}},\textbf{x}^l_{k\{2\}},...,\textbf{x}^l_{k\{C\}}\}$ where $C$ is the size of the window and the number of vocabularies is denoted by $|\textbf{V}_k^l|$. The hidden layer $\textbf{h}_k^l$ is N-dimensional and $\textbf{y}_k^l$ represents the output word. The one-hot encoded input vectors are connected to the hidden layer via $\textbf{W}_k^l$ weight matrix and $\textbf{W}_k^{'l}$ associates the hidden layer to the output. We employ \cite{mikolov2013efficient} to compute both weight matrices of $\textbf{W}_k^{l}$ and $\textbf{W}_k^{'l}$. Given surrounding vocabs, Stochastic Gradient Descent maximizes the conditional probability of the output word.\\
		\vspace{-9mm}
		\begin{figure}[H]
			\centering
			\includegraphics[width=0.49\linewidth]{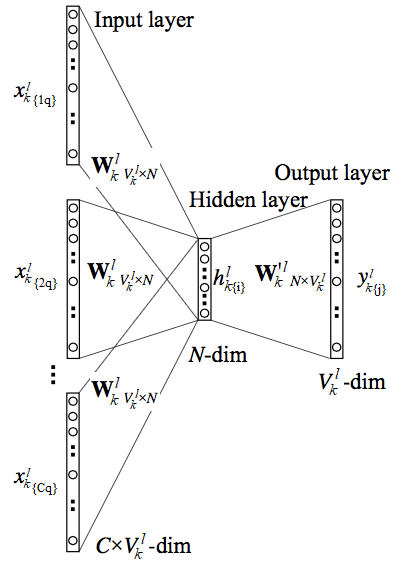}
			\vspace{-3mm}
			\caption{CBOW}
			\label{fig:cbow}
			\vspace{-4mm}
		\end{figure}
		\noindent The hidden layer output is the average of one-hot input vectors that utilize the slab-based weights of $\textbf{W}_k^l$ (Eq. \ref{Eq.hl}).
		\vspace{-2mm}
		\begin{equation}
		\textbf{h}_k^l=\frac{1}{C} \times \textbf{W}_k^l.(\sum_{i=1}^{C} \textbf{x}^l_{k\{i\}})
		\label{Eq.hl}
		\vspace{-2mm}
		\end{equation}
		We also employ Eq. \ref{Eq.ukj} to calculate the input from the hidden layer to every node in the output layer. Here ${\textbf{V}_{k\{\textbf{W}^l_{k\{j\}}\} }^{'l}}$ is the $j^{th}$ column of the output matrix $\textbf{W}_k^{'l}$.
		\vspace{-2mm}
		\begin{equation}
		\textbf{u}^l_{k\{j\}}={\textbf{V}_{k\{\textbf{W}^l_{k\{j\}}\} }^{'l}}^T.\textbf{h}_k^l
		\label{Eq.ukj}		
		\vspace{-2mm}
		\end{equation}
		Finally, we can apply the soft-max function on $\textbf{u}^l_{k\{j\}}$ to attain the output layer $\textbf{y}^l_{k\{j\}}$ (Eq. \ref{Eq.softm}).\\
		\vspace{-2mm}
		\begin{equation}
		\textbf{y}^l_{k\{j\}}=\frac{Exp(\textbf{u}^l_{k\{j\}})}{\sum_{j^{'}=1}^{\textbf{V}^l_k}Exp(\textbf{\textbf{u}}^{'l}_{k\{j\}})}
		\label{Eq.softm}
		\vspace{-2mm}
		\end{equation}
		Given slab $k$ in the facet $l$ and relying on embedding weights ($\textbf{W}_k^{l}$), $\vec{\textbf{v}^l_k}_{\{i\}}$ can denote the embedded vector for each word $i$ in the hidden layer. The \textit{cosine} function (Eq. \ref{cosinesim}) can determine the slab based similarity between each word pair ($i,j$).
		\vspace{-2mm}
		\begin{equation}
		\label{cosinesim}
		\begin{split}
		\mathcal{S}_{Cosine}(\vec{\textbf{v}^l_k}_{\{i\}},\vec{\textbf{v}^l_k}_{\{j\}})=\frac{\vec{\textbf{v}^l_k}_{\{i\}}.\vec{\textbf{v}^l_k}_{\{j\}}}{|\vec{\textbf{v}^l_k}_{\{i\}}|\;\times\;|\vec{\textbf{v}^l_k}_{\{j\}}|}\quad\quad\quad\quad\quad\quad\;
		\vspace{-10mm}
		\end{split}
		\end{equation}
		The vocabulary corpus can be denoted by $V_k^l$ where $\vec{\textbf{v}^l_k}_{\{i\}}$ and $\vec{\textbf{v}^l_k}_{\{j\}}\;$ are the subsets of slab-based vector $\vec{\textbf{V}}_k^l$.
		To continue, we invoke two attributes to better infer the correlation intensity between each pair of words ($i,j$) in temporal slabs.
		\vspace{-6mm}
		\begin{itemize}
			\item \textit{Level} ($\mathcal{S}_{level}(l,\vec{\textbf{v}_i},\vec{\textbf{v}_j})$) explains how extended each pair of words correlate together in all the temporal slabs of a single latent facet of $l$.
			\item \textit{Depth} ($\mathcal{S}_{depth}(l,\vec{\textbf{v}_i},\vec{\textbf{v}_j})$) infers how the pair of words correlate in each slab, while hierarchically impacted by parent temporal dimension(s).
			\vspace{-3mm}
		\end{itemize}
		Eq. \ref{CosineSimLevel} formalizes the level-wise similarity between the vectors of the words $i$ and $j$, where $\tilde{\textbf{A}_k^l}$ is the normalized accuracy of the analogy test for slab $k$ in dimension $l$.
		\vspace{-2mm}
		\begin{equation}
		\label{CosineSimLevel}
		\mathcal{S}_{level}(l,\vec{\textbf{v}_i},\vec{\textbf{v}_j})=\sum_{k\;\in\;l}\tilde{\textbf{A}_k^l}\times\; \mathcal{S}_{Cosine}(\vec{\textbf{v}^l_k}_{\{i\}},\vec{\textbf{v}^l_k}_{\{j\}})
		\vspace{-2mm}
		\end{equation}
		
		Similarly, Eq. \ref{CosineSimDepth} shows the depth similarity between the hidden layer vectors of the pair ($i$,$j$). Note that we propose two static methods for the latent facet class. Suppose that the hour facet is directly impacted by the day dimension $z_h \subset z_d$, in this case, if the current facet in the loop is the day, $l.child()$ can return the child facet, the slabs from hour factor. In contrast, if the hour is the current level, $l.parent()$ can obtain the pointer to the day temporal dimension.
		\vspace{-2mm}
		\begin{equation}
		\label{CosineSimDepth}
		\begin{split}
		\mathcal{S}_{Depth}(l,\vec{\textbf{v}_i},\vec{\textbf{v}_j})=\sum_{k\;\in\; \tau_l^b}\sum_{q\;\in\; \tau_{l.child()}^b} \tilde{\textbf{A}}_q^{l.child()}\quad\quad\quad\quad\\\times\; \mathcal{S}_{Cosine}(\vec{\textbf{v}}^{\;l.child()}_{q\{i\}},\vec{\textbf{v}}^{\;l.child()}_{q\{j\}})
		\end{split}
		\vspace{-2mm}
		\end{equation}
		Here $\tau_l^b$ and $\tau_{l.child()}^b$ are the set of unifacet temporal slabs that are respectively associated with the current level ($l$) and its child dimension. Accordingly, $\tilde{\textbf{A}}_q^{l.child()}$ denotes the normalized accuracy of the analogy test for slab $q$ of the child facet for $l$ dimension. Here the cosine similarity is computed for the hidden layer vector representation of the words ($v_i,v_j$) in two layers of the child ($l.child()$) influenced by the parent ($l$). Note that the similarity feature for all the slabs of the child latent factor $l.child()$ (denoted by $q$), are impacted by corresponding slab $k$ in the parent latent factor (i.e. $l$). Note that Eq. \ref{CosineSimDepth} only comprises two latent facet, that is the current layer $l$ and its direct child $l.child()$. However, given any temporal dimension, the depth property should include the effect of children facets. To this end, we generalize Eq. \ref{CosineSimDepth} to Eq. \ref{CosineSimDepthGenerlized} to recursively call the depth property toward the leaf nodes ($l.child()\neq\;Null$).
		\vspace{-3mm}  
		\begin{equation}		
		\mathcal{S}_{Depth}(l,\vec{\textbf{v}_i},\vec{\textbf{v}_j})=\begin{cases}
		\sum_{q\;\in\; \tau_l^b} \tilde{\textbf{A}}_q^l\times\; \mathcal{S}_{Cosine}(\vec{\textbf{v}}^{\;l}_{q\{i\}},\vec{\textbf{v}}^{\;l}_{q\{j\}}) \\\quad \quad \quad\;\; \;+\; \mathcal{S}_{Depth}(l.child(),\vec{\textbf{v}_i},\vec{\textbf{v}_j}) \\\quad\quad\quad\quad\quad\quad\; if\;l.child()\neq\;Null\\\\
		\sum_{q\;\in\; \tau_l^b} \tilde{\textbf{A}}_q^l\times\; \mathcal{S}_{Cosine}(\vec{\textbf{v}}^{\;l}_{q\{i\}},\vec{\textbf{v}}^{\;l}_{q\{j\}})\\\quad\quad\quad\quad\quad\quad\quad\quad\quad\quad\;\; otherwise
		\end{cases}
		\label{CosineSimDepthGenerlized}
		\vspace{-1mm}
		\end{equation}
		Eventually, the correlation intensity ($[-1,1]$) between each pair of words ($i,j$) can be collectively evaluated by both the \{level and depth\}-wise attributes, Eq. \ref{CosineSimTotal}:
		\vspace{-2mm}
		\begin{equation}
		\label{CosineSimTotal}
		\overline{\mathcal{S}}_{Cosine}(\vec{\textbf{v}_i},\vec{\textbf{v}_j})=\sum_{l\;\in\; \mathbb{\textbf{T}}} (\mathcal{S}_{level}(l,\vec{\textbf{v}_i},\vec{\textbf{v}_j})+\mathcal{S}_{Depth}(l,\vec{\textbf{v}_i},\vec{\textbf{v}_j}))
		\vspace{-1mm}
		\end{equation}
		After computing the similarity between each pair of words, we obtain $B^{TCBOW}$ as a $|V|\times|V|$ matrix, where each row $i$ is associated with a single vocabulary $v_i$ and the resulting word vector $\vec{v}_i$ can represent the similarity between $v_i$ and other words. $\vec{V}^{TCBOW}$ denotes the set of the word vectors in $B^{TCBOW}$ grid.\\
		Inherently, the dimension of vectors in $\vec{V}^{TCBOW}$ equates to the number of words (i.e. $|V|$) and turns out to be much more than $|d|$, which is the dimension of vectors in the hidden layer. Due to complexity, such a high dimension can negatively affect efficiency. 
		To address this challenge, we propose $\vec{V}^C$ which collectively computes the word vectors based on each of the slabs in all the latent facets. To compute these collective word vectors, we take into account the effect of the word vectors of $\vec{V}_k^l$ of each slab $S_k^l$ by multiplying those vectors into their normalized analogy accuracy $\tilde{A}_k^l$.\\
		We compute the collective word vector $\vec{v}_i^{C}$ for each word $v_i$ using two attributes of \textit{level} and \textit{depth}. Unlike $\mathcal{V}_{level}(l,i)$ which only includes the slabs for the current latent facet (Eq. \ref{VectorLevel}), the depth property $\mathcal{V}_{depth}(l,i)$ hierarchically considers the effects from all the parent latent facets.
		\vspace{-2mm}
		\begin{equation}
		\label{VectorLevel}
		\mathcal{V}_{level}(l,i)=\sum_{k\;\in\;l}\tilde{\textbf{A}_k^l}\times\; \vec{v}_{k\{i\}}^{\;l}
		\vspace{-2mm}		
		\end{equation}
		$\mathcal{V}_{level}(l,i)$ considers the slabs in the same latent facet. Similarly, Eq. \ref{VectorDepthGenerlized} calculates the depth property.
		\vspace{-2mm}
		\begin{equation}	
		\mathcal{V}_{Depth}(i)=\begin{cases}
		\sum_{q\;\in\; \tau_l^b} \tilde{\textbf{A}}_q^l\times\; \vec{v}_{k\{i\}}^{\;l} \;+\; \mathcal{V}_{Depth}(l.child(),i)\\ \quad\quad\quad\quad\quad\quad\quad\quad\quad\;\;  if\;l.child()\neq\;Null\\\\
		\sum_{q\;\in\; \tau_l^b} \tilde{\textbf{A}}_q^l\times\; \vec{v}_{k\{i\}}^{\;l} \quad \quad  \quad\quad \quad \;  otherwise
		\end{cases}
		\label{VectorDepthGenerlized}
		\end{equation}
		Here $l$ is the current layer and $i$ is the index of words for which the collective vector is computed. Like Eq. \ref{CosineSimDepthGenerlized}, the depth property behaves recursively.
		\vspace{-2mm}
		\begin{equation}
		\label{VectorTotal}
		\vec{v}_i^{C}=\sum_{l\;\in\; \mathbb{\textbf{T}}} (\mathcal{V}_{level}(i)+\mathcal{V}_{Depth}(i))
		\vspace{-2mm}
		\end{equation}
		Ultimately, the final collective word vector for each word $v_i$, denoted by $\vec{v}_i^{C}$, will be attained by the summation of level and depth functions (Eq. \ref{VectorTotal}).
		\vspace{-3mm} 
		\subsubsection{Generating tweet vectors}
		\vspace{-1mm}
		Given the word vectors that are constructed by the temporal embedding model (Collective), we now need to generate \textit{tweet vectors}.
		\textit{Summation} and \textit{Averaging} are two simple but effective approaches to combine word vectors in each tweet and obtain the outcome tweet vector.
		While the summation approach generates vectors with bigger values and augments the computation time, the average method places the resulting vector between input vectors, which can better represent the blending. The Tweet vector is computed by merging the vectors of the comprising word (Eq. \ref{SumAvgVector}).\\
		\vspace{-2mm}
		\begin{equation}
		\label{SumAvgVector}
		\begin{split}
		\vec{m}_i^{\; Avg}=\frac{\sum_{j=1}^{|m_i|}\vec{v}_{m_i[j]}}{|m_i|},\quad \vec{m}_i^{\; Sum}=\sum_{j=1}^{|m_i|}\vec{v}_{m_i [j]}
		\end{split}
		\vspace{-4mm}
		\end{equation}
		Here, $|m_i|$ denotes the number of words in each short-text $m_i$ and $\vec{v}_{m_i [j]}$ constitutes the word vector for the $j^{th}$ word in $m_i$.
		Tweet vectors can represent the comprising word vectors. However, short-text instances might refer to different concepts when context differs. Therefore, understanding of the concept(s) to which the tweets correspond matters in recognition of the authors' orientations.		
		To this end, we need to dynamically discover the concepts that are shared among each group of tweets. Hence, any upcoming unclassified tweet can be conceptually grouped into one of the existing tweet clusters. We utilize two popular clustering methods of DBScan \cite{schubert2017dbscan} and K-Medoids \cite{cheng2015contextual}, which differ in nature.  Where DB-Scan detects the densely grouped tweets, K-medoids discover the outliers, that have been cast-out by the DB-Scan algorithm. 
		Nevertheless, we employ the well-known \textit{Euclidean distance} to measure the space between cluster points (Eq. \ref{EuclideanDistance}).\\
		\vspace{-3mm}
		\begin{equation}
		\label{EuclideanDistance}
		Distance_{Euclidean}(\vec{v}_i,\vec{v}_j)=\sqrt{\sum_{p=1}^{d}(\vec{v}_{i\{p\}}-\vec{v}_{j\{p\}})}
		\vspace{-2mm}		
		\end{equation}
		Here $d$ denotes the dimension of word vectors and $p$ indicates the index of any word vector (e.g. $\vec{v}_{i}$). Nevertheless, given the list of exploited clusters, we can present each tweet $m_i$ with the \textit{tweet concept vector}. The new vector lists the dissimilarities between each tweet $m_i$ and each of the concepts that are the center point of each cluster, Eq. \ref{ConceptVector}.
		\vspace{-2mm}
		\begin{equation}
		\label{ConceptVector}
		\forall_{j=1}^{|C^{f}|}\;\;\vec{m}_{i}^{\;f}[j]=Distance_{Euclidean}(\vec{m}_i,\vec{c}^{\;f}_j)
		\vspace{-2mm}
		\end{equation}
		Here, $|\mathbb{C}^{f}|$ shows the number of clusters that are extracted using any clustering method of $f$. Where $\vec{m}_{i}^{\;f}$ shows the tweet concept vector that is computed using $f$, and $\vec{c}^{\;f}_j$ symbolizes the center tweet vector of the $j^{th}$ cluster that is extracted using the same clustering model (i.e. $f$). Finally, $\vec{m}_{i}^{\;f}[j]$ denotes the $j^{th}$ entry of the $\vec{m}_{i}^{\;f}$. Furthermore, $\mathbb{C}$ as the chosen set of clustering models ($f\;\in\mathbb{C}$) can collectively include two features from the type of clustering (k-medoid and DB-Scan), and whether the tweet vectors are constituted from summation or average of comprising word vectors. For instance, $Sum-DB$ can specify an $f$ clustering model where the tweet vectors are constructed by the summation of word vectors and the employed clustering model is DB-Scan. It is noteworthy that the Tweet concept vectors tend to impose the smaller grid of $\mathbb{R}^{|\mathbb{C}^{f}|}$.
		\vspace{-3mm}
		\subsubsection{Generating author vectors} 
		\label{Generating author vectors}
		\vspace{-1mm}
		We explain in detail the approach we take to associate each author with content and concept vectors. Recall that each tweet vector is constructed via the merging of the vectors of the words it comprises. Similarly, we can apply the summation or average operators on tweet vectors to obtain the author's content vector. Let $m_j$ be a tweet from the set of tweets composed by the author $n_i$ ($m_j \in M_i$) where $\vec{m}_j$ denotes the vector for $m_j$. It is easy to see that the sum and average vectors for author $n_i$ can be computed using Eq. \ref{SumAvgAuthorVector}.\\
		\vspace{-2mm}
		\begin{equation}
		\label{SumAvgAuthorVector}
		\begin{split}
		\vec{n}_i^{\; Content-Sum}=\sum_{m\;\in\;M_i}\vec{m},\quad \vec{n}_i^{\; Content-Avg}=\frac{\sum_{m\;\in\;M_i}\vec{m}}{|M_i|}
		\end{split}
		\end{equation}
		$\vec{n}_i^{\; Content-Sum}$ and $\vec{n}_i^{\; Content-Avg}$ respectively represent the sum and average author content vectors.
		Rather than considering operational functions (sum and average), let us consider statistical approaches. Here, we aim to predict the vector of an unobserved tweet that is deemed to be semantically aligned toward the preferences of the author $n_i$. While topic-modeling \cite{cao2017you} seems to be the best solution for document generation, we extend the embedding model of the tweet vectors by the K-Fold statistics. Let $d$ be the dimension of the tweet content vector. Hence, we can conveniently discover the maximum probability for each value to be assigned to the corresponding item in the author's content vector. To this end, we partition the range [-1,+1] into $\varsigma$ number of bins (In here $\varsigma=10$). For each index in the author vector, we consider the bin which contains the values from the majority of the tweet vectors on the same index. The selected bin will represent the weight with the highest probability for the current index in author vector.
		\vspace{-4mm}
		\begin{figure}[H]
			\centering
			\includegraphics[scale=0.5]{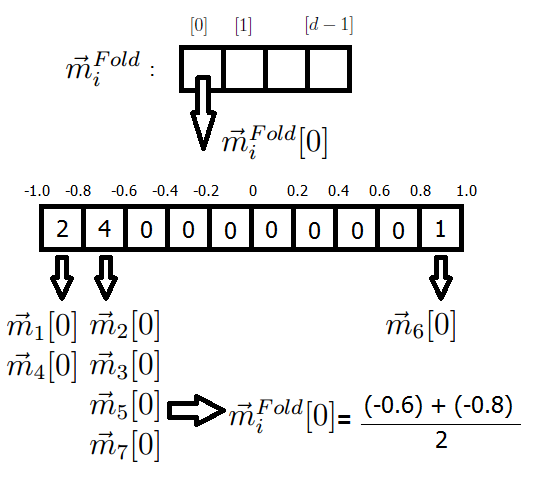}
			\vspace{-4mm}
			\caption{(Fold)} 
			\label{Img:Fold}
			\vspace{-4mm}
		\end{figure}
		As instantiated in Fig. \ref{Img:Fold}, the bin [-0.8,-0.6] reports the values from four tweets and demonstrates that with more than 50\% chance, the index zero of the author vector will be assigned between -0.8 and -0.6. For the bins with equal numbers, we utilize a linked list and distribute the prediction chance accordingly. Similar to the author's content vector, we can aggregate the tweet concept vectors of each author to construct the author's concept vector ($\vec{n}_i^{\; Concept}$). The concept vectors are low in dimension and dynamically carry the knowledge of concepts through clustering. However, content vectors come with less complexity and include the information of the tweets which are possibly ignored in DB-Scan or misclassified by K-medoid. The eventual author similarity should include the impacts from both content and concept vectors (Eq. \ref{UserVector0}) that is adjusted by $\alpha$.
		\vspace{-2mm}
		\begin{equation}
		\label{UserVector0}
		\forall_{i=1}^{|U|}\forall_{j=1}^{|U|},\;X^{Total-\alpha}_{ij}=\alpha\times X^{Concept}_{ij}+(1-\alpha)\times X^{\;Content}_{ij}
		\vspace{-1mm}
		\end{equation}
		Because of the difference in dimensionality, it is not inherently feasible to consolidate content and concept vectors of the authors. Given $U$ as the set of authors, we should compute the similarity between each pair of authors to build the distinctive correlation matrices using concept and content vectors, respectively denoted by $X^{Concept}$ and $X^{Content}$. The impact from correspondent matrices can be merged by $\alpha$ to maximize the final performance.
		\vspace{-4mm}
		\subsection{Online Phase}
		\vspace{-2mm}
		In the online phase, we aim to mine a set of authors that are highly correlated to the query author $n_q$. Two duties are undertaken in online phase: Including query author and extracting stack-wise maximum spanning trees.
		\vspace{-3mm}
		\subsubsection{Including query author}
		\vspace{-1mm}
		This duty is divided into two tasks of \textit{generating query author vectors} and \textit{computing query author contextual similarities} that are quite similar to what we explained in Section \ref{offline_phase}. Firstly, we need to generate the author vectors for the query author $n_q$. This is especially necessary for the cold start authors. Because by posting a single tweet, the conceptual alignment of the author may evolve.\\
		Given the set of current tweets ($M_q=\{m_1,m_2,...,m_r\}$) belonging to $n_q$, we can generate corresponding tweet vectors $\vec{M}_q=\{\vec{m}_1,\vec{m}_2,...,\vec{m}_r\}$ using precomputed $\vec{V}^C$. This step is not time-consuming as the language model is already generated in the offline phase.\\
		\textit{What is the usage for Trigger?} Trigger follows frequent intervals to continuously rebuild the slabs and subsequently construct the vector representations. This is especially useful to include the tweets of new authors where it can partially affect the embedding results. Therefore, the new tweets will be included in the embedding process in the offline phase as soon as the trigger is released.\\
		Using the tweet vectors we can easily retrieve the $n_q$'s content vector $\vec{n}^{Content}_q$. Correspondingly, we need to find the distance between each tweet $m_i \in M_q$ and the cluster centroids which results in the set of tweet concept vectors $\vec{M}^f_q=\{\vec{m}^f_1,\vec{m}^f_2,...,\vec{m}^f_r\}$. Here $f$ denotes the selected clustering approach. Accordingly, the author concept vector $\vec{n}^{Concept}_q$ can be computed by averaging of the vectors in $\vec{M}^f_q$. Given the content $\vec{n}^{Content}_q$ and concept $\vec{n}^{Concept}_q$ vectors of the query author $n_q$, we can respectively update $X^{Content}$ and $X^{Concept}$ author similarity matrices which are accomplished through measuring the similarity between $n_q$ and others. Eventually, graph $G=(\mathbb{N},\mathbb{L})$ can represent the authors weighted graph, where $\mathbb{N}$ is the set of nodes (authors) and $\mathbb{L}$ denotes the set of undirected edges with similarity weights.
		\subsubsection{extracting query author subgraph} 
		\label{extracting-query-author-subgraph} 
		\vspace{-1mm}	
		We now aim to exploit the subgraphs with highly correlated authors which further comprises the query author $n_q$. In general, we can address the challenge through Lemma \ref{lemma_nq}.  Inspired by the Lemma \ref{lemma_nq}, we devise the Stack-Wise Maximum Spanning Tree (SW-MST) approach (Algorithm \ref{alg6}) to calculate the MST for each of the highly correlated subgraphs in $G$. As algorithm \ref{alg6} shows, we firstly push the edges into the empty stack $S$ in ascending order, where the links with the lower weight are pushed downward. Correspondingly, we initiate an empty graph $G'=(\mathbb{N}',\mathbb{L}')$ to store the resulting spanning trees. To continue, we iteratively pop the edges from the stack and add them to $\mathbb{L}'$ and append the corresponding nodes to $\mathbb{N}'$. We repeat the process until every $n_i \in \mathbb{N}$ is added to the $\mathbb{N}'$. The $G'$ will finally include a set of maximum spanning-trees. In other words, the algorithm \ref{alg6} firstly extracts distinctive subsets of the graph in the form of maximal cliques and subsequently exploits an MST out of the cliques. Finally, each exploited MST can represent a highly correlated author subgraph.
		\vspace{-2mm}
		\begin{lemma}
			\textit{Linking a set of highly correlated authors to the query author $n_q$ can be facilitated by the inner-author edge weights.}
			\label{lemma_nq} 
			\vspace{-2mm}	
		\end{lemma}
		\begin{proof}
			Given a fully connected weighted graph which represents the weight of contextual similarity between author vectors, the subgraph $\tilde{g}_q$ comprising $n_q$ will result in the maximum spanning tree with the biggest average edge weight. With this logic, the node $n_q$ will be highly correlated to every node in $\tilde{g}_q$ via either of the direct or indirect link(s).
			\vspace{-2mm}
		\end{proof}
		\vspace{-3mm}
		\begin{algorithm}[H]
			\caption{Stack-Wise Max. Spanning Tree (SW-MST)}
			\label{alg6}
			\textbf{Input:} $G$\\
			\textbf{Output:} $G', S$
			\begin{algorithmic}[1]
				\STATE $\mathbb{N}'=\emptyset, \mathbb{L}'=\emptyset, \mathbb{N}''= \mathbb{N}, \mathbb{L}''= \mathbb{L}, S=\emptyset $
				\WHILE{$\mathbb{L}'' \neq \emptyset$}
				\STATE $l=Min(\mathbb{L}'')$
				\STATE $S.push(l)$
				\STATE $\mathbb{L}''.remove(l)$
				\ENDWHILE
				\WHILE{$\mathbb{N}'' \neq \emptyset$}
				\STATE $\;l=S.pop()$
				\STATE $\;\mathbb{L}'.append(l)$
				\IF{$l[0] \notin \mathbb{N}'$}
				\STATE $\mathbb{N}'.append(l[0])$
				\ENDIF
				\IF{$l[1] \notin \mathbb{N}'$}
				\STATE $\mathbb{N}'.append(l[1])$
				\ENDIF
				\IF{$l[0] \in \mathbb{N}''$}
				\STATE $\mathbb{N}''.remove(l[0])$
				\ENDIF
				\IF{$l[1] \in \mathbb{N}''$}
				\STATE $\mathbb{N}''.remove(l[1])$
				\ENDIF
				\ENDWHILE
				\STATE $G'=(\mathbb{N}',\mathbb{L}')$
				\STATE return $G'$, $Avg(\mathbb{L}')$
			\end{algorithmic}
		\end{algorithm}	
		\vspace{-9mm}	
		\section{Experiment}
		\label{experiments}
		\vspace{-1mm}
		We conducted extensive experiments on a real-world twitter dataset \cite{Hosseini2014} to evaluate the performance of our model in short-text author linking. 
		We ran the experiments on a server with 4.20GHz Intel Core i7-7700K CPU and 64GB of RAM. The codes are available to download \footnote{\url{https://sites.google.com/view/temporal-textual-embedding}}.
		\vspace{-5mm}
		\subsection{Data}
		\label{dataset}
		\vspace{-1mm}
		Our Twitter dataset \cite{Hua2017} includes 8 million English tweets in Australia, collected via \textit{Spritzer Twitter Feed}. The sampling was done at various times of the day for a complete year. We then used \textit{Twitter API} to select approximately $4K$ users from streaming tweets and retrieved up to 1000 records from their Twitter history. Finally, we attained $\approx 1M$ geo-tagged \textit{tweets} which are all composed in Australian territory. The dataset contains 305K vocabs and is made of 65M collocations.
		\vspace{-3mm}
		\subsubsection{Baselines}
		\label{baselines}
		\vspace{-1mm}
		The baselines in computing of the similarity weights between authors are listed as follows. Note that the author's similarities can be computed by measuring the similarity between author vectors.
		\vspace{-2mm}
		\begin{itemize}
			\item $SoulMate_{Concept}$: 
			As explained in Section \ref{Generating author vectors}, this method renders the authors with the closeness of their tweets to each of the concepts.
			\item $SoulMate_{Content}$: this embedding approach \cite{mikolov2013efficient} obtains the tweet vectors and then combines them to form author content vectors.
			\item $SoulMate_{Joint}$: this model regulates $\alpha$ to combine the author's similarities through concept and content vectors (Section \ref{Generating author vectors}).
			\item Temporal Collective: this model computes the collective word vector through multi-facet temporal embedding \cite{hosseini2017leveraging}\cite{hosseini2019teals} and then enriches the textual contents of each author by replacing each word with its top $\zeta$ most similar words. Finally, TF-IDF can measure the textual similarity between authors.
			\item CBOW Enriched: this model uses CBOW \cite{mikolov2013efficient}\cite{Kuzi:2017:QEE:3077136.3080660} to produce the distributed representation of the words. Given the enriched textual contents of the authors, the model employs the Jaccard coefficient to compute the textual similarities.
			\item Document Vector: this model \cite{beel2016paper} computes the similarities between authors using TF-IDF statistics.
			\item Exact Matching: this straightforward baseline exactly matches the short-text contents of the authors.
			\vspace{-4mm}		
		\end{itemize}
		\subsection{Effectiveness}
		\vspace{-1mm}		
		\subsubsection{Basic Comparison of Vector Space Models}
		\vspace{-1mm}		
		\label{Comparing-vector-representation-models}
		In this section, we firstly apply the Google word analogy task \cite{Mikolov2013a} to compare the effectiveness and efficiency of the vector representation models. The vector representation baselines (Section \ref{Sec:Word_Embedding}) are four-fold: SVD \cite{sugawara2015context}, Skip-gram \cite{Mikolov2013a}, CBOW \cite{mikolov2013efficient}, and GloVe \cite{Pennington2014}. \textit{SVD-15:15000} limits the word pair co-occurrences between 15 and 15000. Also, the numerical extension in GloVe-30 highlights the number of training epochs. The analogy test aims to discover the model that on the one hand, suits best to the short-text noisy contents, and on the other hand, is the best candidate for time-aware embedding. The test includes $\approx 20K$ syntactical and semantical questions like ``a is to b as c is to ?'', where each competitor suggests a word to alter the question mark.
		\vspace{-7mm}		
		\begin{figure}[H]	
			\centering			
			\begin{subfigure}{0.49\linewidth}
				\centering
				\includegraphics[width=1\textwidth]{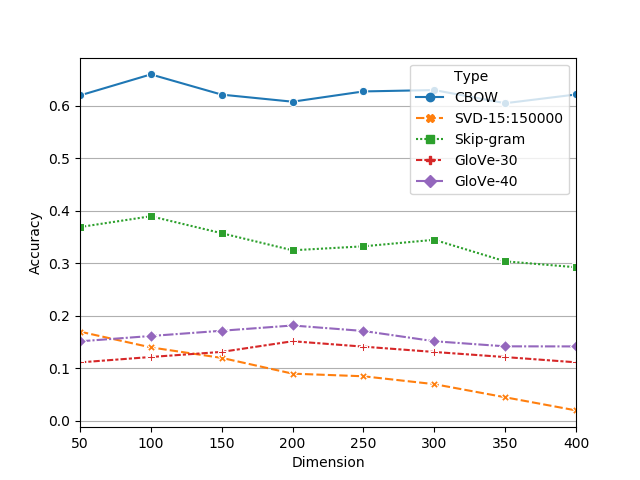}
				\caption{Accuracy}
				\label{fig:TAccU}
			\end{subfigure}
			\centering
			\begin{subfigure}{0.49\linewidth}
				\centering
				\includegraphics[width=1\textwidth]{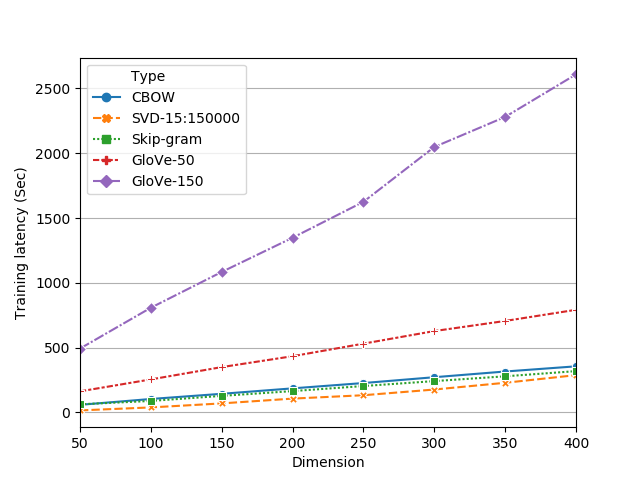}
				\caption{Efficiency}
				\label{fig:TTL-old}
			\end{subfigure}
			\hfill
			\vspace{-2mm}
			\caption{Performance of the vector space models}
			\label{fig:UnderStudyAccuracy}
			\vspace{-3mm}
		\end{figure}
		Fig. \ref{fig:TAccU} reports the accuracy of the analogy task on the twitter dataset (Section \ref{dataset}), where the dimension varies. Our dataset suffices the words for only $\approx 7K$ questions, resulting in lower numbers. The CBOW model overpasses all rivals and SVD performs the least as it lacks the training phase. Conversely, the CBOW as the most noise-resistant model surpasses skip-gram because it better involves the context of the words in the training procedure. Finally, notice that excessive noise in microblog contents leads to a sparse and oversize co-occurrence matrix which significantly reduces the performance of the GloVe model.\\
		Our online author linking framework must handle millions of the short-text contents, where the vector representation module forms the underlying time-aware module. So as illustrated in Fig. \ref{fig:TTL-old}, we compare the efficiency of vector space methods. We notice that due to the lack of training procedure the temporal latency of the SVD model is the least. Furthermore, for the models with training, the CBOW and skip-gram closely gain the highest efficiencies. However, the GloVe models take the highest time in training which is naturally due to the huge size of the input co-occurrence matrix. Hence, we conclude that CBOW is better than other models in both effectiveness and efficiency.
		\vspace{-3mm}
		\begin{table}[H]
			\centering
			\caption{Precision of author similarity in subgraph mining}
			\vspace{-3mm}
			\def\arraystretch{1.5}
			\small
			\begin{tabular}{l|c|c|}
				\cline{2-3}
				& \multicolumn{2}{c|}{Score Distribution}                                                                                                               \\ \hline
				\multicolumn{1}{|c|}{\begin{tabular}[c]{@{}c@{}}\end{tabular}} & \begin{tabular}[c]{@{}c@{}} textual$\uparrow$\\ conceptual$\uparrow$\end{tabular} & \begin{tabular}[c]{@{}c@{}}textual$\downarrow$\\ conceptual$\uparrow$\end{tabular} \\ \hline
				\multicolumn{1}{|l|}{$SoulMate_{Concept}$}                                                                                      & 0.07                                 &    0.30                                    \\ \hline
				\multicolumn{1}{|l|}{$SoulMate_{Content}$}                                                                                      & 0.43                                    & 0.05                                    \\ \hline
				\multicolumn{1}{|l|}{$SoulMate_{Joint}$}                                                                                      & \textbf{0.67}                                     & \textbf{0.32}                                       \\ \hline
				\multicolumn{1}{|l|}{Temporal Collective}                                                                                      & 0.63                                     & 0.01                                       \\ \hline
				\multicolumn{1}{|l|}{CBOW Enriched}                                                                                      & 0.48                                     & 0                                        \\ \hline
				\multicolumn{1}{|l|}{Document Vector}                                                                                      & 0.21                                     & 0                                        \\ \hline
				\multicolumn{1}{|l|}{Exact Matching}                                                                                      & 0.39                                     & 0.01                                        \\ \hline
			\end{tabular}
			\label{FinalComparison}
			\vspace{-3mm}
		\end{table}
		\subsubsection{Comparison of the Author Subgraph Mining Methods}
		\label{Effectiveness of author similarities}
		In this part, we compare our approaches to linking authors with well-known competitors (Section \ref{baselines}). As the first step, the author's similarity matrix of each baseline model can establish the author weighted graph. We propose three algorithms to calculate author similarities ($SoulMate_{Concept}$, $SoulMate_{Content}$, and $SoulMate_{Joint}$). Eventually, given the originated weighted graph, each model can employ SW-MST algorithm (Section \ref{extracting-query-author-subgraph}) to acquire the final author subgraphs, as Maximum Spanning Trees (MST).\\
		\textbf{Benchmark- } Since the authors within each spanning tree should exceedingly correlate, as Table \ref{FinalComparison} shows, we evaluate the baselines through assessing the similarity between authors in the same exploited subgraphs. To this end, we first obtain the set of MSTs out of $G'$ (output of SW-MST) which comprises any of 50 arbitrarily chosen authors. We then pick top 5 MSTs with at least 5 nodes that possess the highest average edge weights. Finally, given the top 10 most similar tweets from each pair of authors in the selected MSTs, we consider the votes of 5 local (Australian) experts. The possible votes are defined as follows:
		\vspace{-2mm}
		\begin{itemize}
		\item score 0: neither textually or conceptually similar.
		\item score 1: minor textual and conceptual similarity.
		\item score 2: high textual and conceptual similarity.
		\item score 3: minor textual but high conceptual similarity.	
		\vspace{-5mm}
		\end{itemize}
		Subsequently, we compute the average of the votes given to each pair of tweets and round it to the nearest lower integer. We then count the tweet pairs with the scores of 2 and 3 for each one of the author similarity calculation methods. The precision metrics are then calculated by dividing the number of 2 and 3 scores (admitted by the average of experts' votes) by the total number of selected tweet pairs in subgraphs. $SoulMate_{Concept}$ is devised to detect the conceptual similarities, where the textual relevance is minor. Moreover, the $SoulMate_{Content}$ can trace the textual and conceptual relevance. However, since $SoulMate_{Joint}$ combines both modules through parameter adjustment ($\alpha$=0.6), it gains the highest votes for both conditions. It is interesting to see that where the textual similarity between short-text contents is very low and all textual models including Temporal Collective, CBOW, Document Vector and Exact matching fail (perform less than 2\%), $SoulMate_{Concept}$ can detect the semantic correlation between authors by 30\%. Notice that the higher the number of exploited concepts, and the better the clustering models, $SoulMate_{Concept}$ model can gain a better precision. Table \ref{FinalComparison} shows, where the textual similarity is low, $SoulMate_{Concept}$ can still find (precision \%30) the conceptually relevant author pairs. Conversely, $SoulMate_{Content}$ can track textual similarity (\%43). To briefly mention, $SoulMate_{Joint}$ as our final model performs more accurately than other baselines. Note that based on our peripheral experiments on Temporal word embedding (Section \ref{Temporal-word-embedding}), where the accuracy of $\vec{V}^{TCBOW}$ in generating of the tweet vectors is 0.881, the dimension is quite large (the number of words ($|V|$)). Hence, we employ the collective manner $\vec{V}^C$ which offers a lower precision of 0.861 but in contrast, provides a much smaller dimension (the size of hidden layer vectors ($|d|$)).\\
		\vspace{-7mm}
		\subsubsection{Effect of embedding on author content vectors}
		\label{Impact of embedding and aggregation approaches on content vectors}
		\vspace{-1mm}
		We here study the impact of several parameters on the effectiveness of author contents vectors. We compare \textit{CBOW} versus \textit{Collective} that are respectively the best non-temporal and temporal embedding models. To form the tweet vectors, the word vectors can be combined using summation or averaging. The tweet vectors can also form the author's content vectors through various aggregations, such as Average, Summation, and 10 Fold model.\\
		\textbf{Benchmark -} The experts label the top 10 most similar tweets with the mentioned scores that were previously defined in the benchmark part of Section \ref{Effectiveness of author similarities}. We also deploy averaging and rounding to the bottom integer method, as we did in the previous section, to take into account the votes of all the experts. We then consider the computed score as the final score for each selected pair of tweets.
		Inspired by \cite{Rosa2011}, we propose two weighted precision equations of \ref{P_Conceptual} and \ref{P_Textual} to compare the effectiveness of the methods:
		Note that in both equations, the pertinent scores are prefixed by $\rho$ (e.g. the number of items for score 1 is denoted by $\rho_1$).\\
		$\mathcal{P}_{Conceptual}$: The weighted precision formulated in Eq. \ref{P_Conceptual} pays more attention to the pairs with high conceptual but low textual similarity which leads to the high numerical coefficient of 3 for $\rho_3$ and the null significance for $\rho_0$. Here, the precision is normalized by multiplying the sum of the score counts by 3 in the denominator.\\
		\vspace{-3mm}
		\begin{equation}
		\label{P_Conceptual}
		\mathcal{P}_{Conceptual}=\frac{\rho_1+\rho_2\times 2+\rho_3\times 3}{3\times (\rho_0+\rho_1+\rho_2+\rho_3)}
		\vspace{-1mm}
		\end{equation}
		$\mathcal{P}_{Textual}$: As verbalized in Eq. \ref{P_Textual}, both textual and conceptual similarities gain the same importance in $\mathcal{P}_{Textual}$ metric. This enforces the same coefficient of 2 for $\rho_2$ and $\rho_3$.
		\vspace{-2mm}
		\begin{equation}
		\label{P_Textual}
		\mathcal{P}_{Textual}=\frac{\rho_1+(\rho_2+\rho_3)\times 2}{2\times (\rho_0+\rho_1+\rho_2+\rho_3)}
		\vspace{-1mm}
		\end{equation}
		\begin{table*}[t]
			\centering			
			\caption{Weighted precision of user content vectors}
			\vspace{-2mm}
			\def\arraystretch{1.5}
			\tiny
			\begin{tabular}{cl|c|c|c|c|}
				\cline{3-6}
				\multicolumn{1}{l}{}                                    &                                                                                                                        & \multicolumn{4}{c|}{\begin{tabular}[c]{@{}c@{}}Tweet vector\\ combination type\end{tabular}} \\ \hline
				\multicolumn{1}{|c|}{\multirow{2}{*}{Embedding method}} & \multicolumn{1}{c|}{\multirow{2}{*}{\begin{tabular}[c]{@{}c@{}}Author content vector\\ combination type\end{tabular}}} & \multicolumn{2}{c|}{Average}           & \multicolumn{2}{c|}{Summation}                      \\ \cline{3-6} 
				\multicolumn{1}{|c|}{}                                  & \multicolumn{1}{c|}{}                                                                                                  & $\mathcal{P}_{Textual}$            & $\mathcal{P}_{Conceptual}$           & $\mathcal{P}_{Textual}$                  & $\mathcal{P}_{Conceptual}$                   \\ \hline
				\multicolumn{1}{|c|}{\multirow{3}{*}{CBOW}}             & Average                                                                                                                & 0.547              & 0.433             & 0.645                    & 0.511                    \\ \cline{2-6} 
				\multicolumn{1}{|c|}{}                                  & Summation                                                                                                              & 0.547              & 0.433             & 0.645                    & 0.511                    \\ \cline{2-6} 
				\multicolumn{1}{|c|}{}                                  & 10 Fold                                                                                                                & 0.594              & 0.442             & 0.668                    & 0.477                    \\ \hline
				\multicolumn{1}{|c|}{\multirow{3}{*}{Collective}}            & Average                                                                                                                & 0.568              & 0.453             & \textbf{0.652}           & \textbf{0.538}           \\ \cline{2-6} 
				\multicolumn{1}{|c|}{}                                  & Summation                                                                                                              & 0.568              & 0.453             & \textbf{0.652}           & \textbf{0.538}           \\ \cline{2-6} 
				\multicolumn{1}{|c|}{}                                  & 10 Fold                                                                                                                & 0.602              & 0.461             & \textbf{0.676}           & 0.493                    \\ \hline
			\end{tabular}
			\label{tab:ContentVector}
			\vspace{-6mm}
		\end{table*}
		\noindent While among the embedding methods, the time-aware approach (collective embedding) is better than CBOW, the summation works better than average in the aggregation of word vectors. The 10 fold algorithm gains a higher precision for $\mathcal{P}_{Textual}$ in the aggregation of tweet vectors. However, since the 10 Fold approach performs low for $\mathcal{P}_{Conceptual}$, we opt for other aggregation algorithms that support both of the weighted precisions. 
		Since the normalized vector for the summation method is very similar to the average approach, both precisions come with similar results. Regarding author content vectors, while summation and averaging methods are the same, we select the average operator for the lower decimal values and the less computational complexities.
		\vspace{-3mm}		
		\subsubsection{Impact of Short-text Vector Clustering}
		\label{Impact of short-text vector clustering}
		\vspace{-1mm}		
		To extract the concepts from microblog contents, we need to cluster the tweet vectors. We aim to select those thresholds that can maximize the number of exploited clusters (concepts), and simultaneously maintain a satisfying quality. Hence, in this section, we compare the performance of various clustering models. Note that we use $K$ as the number clusters for K-medoids and $\epsilon$ as the radius for DBSCAN. Where the thresholds vary, it is both tedious and time-consuming to test the quality of clusters by human experts. Therefore in this section, we firstly study the cohesion and separation properties of the clusters using two well known methods of the Silhouette score \cite{rousseeuw1987silhouettes} and the Davies-Bouldin index \cite{davies1979cluster}. Subsequently in Section \ref{Impact of zeta on selected}, we limit the range of thresholds first and then ask the experts to evaluate the quality of clusters.
		\vspace{-4mm}
		\begin{figure}[H]
			\centering
			\begin{subfigure}{0.6\linewidth}
				\centering
				\includegraphics[width=1\textwidth]{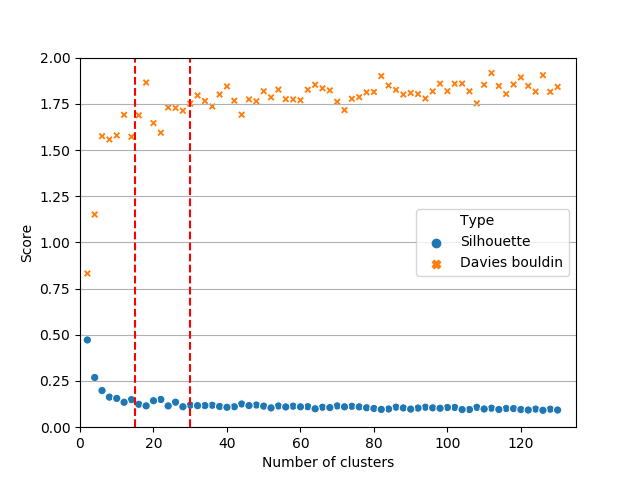}
				\vspace{-6mm}
				\caption{K-Medoid: Quality}
				\label{fig:K_Medoid_Cluster_Score}
			\end{subfigure}
			\centering
			\begin{subfigure}{0.45\linewidth}
				\centering
				\includegraphics[width=1\linewidth]{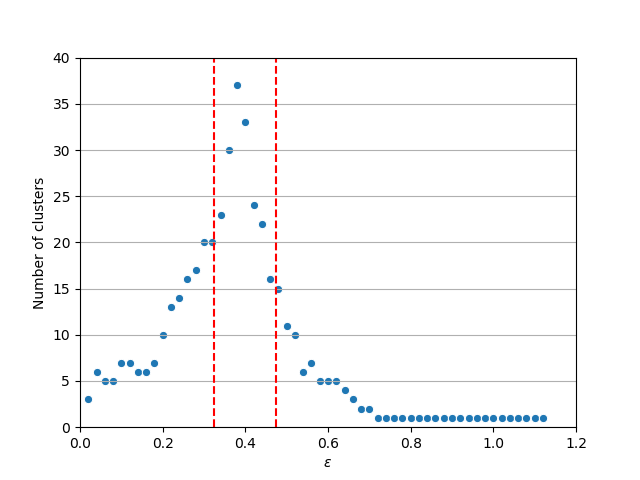}
				\vspace{-6mm}
				\caption{DBSCAN: \# of clusters}
				\label{fig:EPS_Clusters}
			\end{subfigure}`
			\centering
			\begin{subfigure}{0.45\linewidth}
				\centering
				\includegraphics[width=1\textwidth]{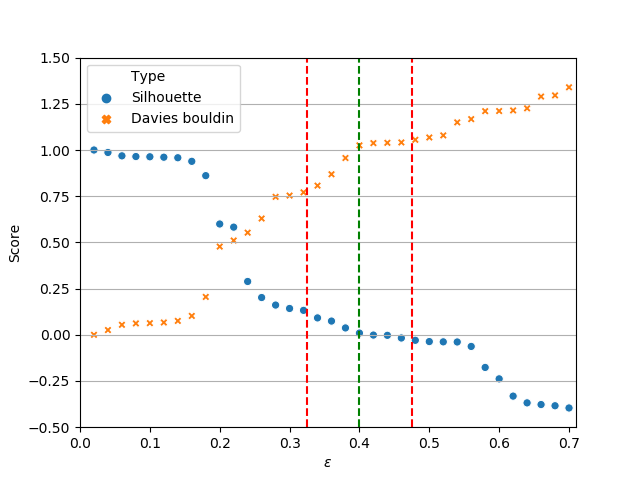}
				\vspace{-6mm}
				\caption{DBSCAN: Quality}
				\label{fig:Siloh_Davis_Eps}
			\end{subfigure}
			\hfill
			\vspace{-3mm}
			\caption{Impact of thresholds on clustering}
			\label{fig:ImpactofClusterNumber}
			\vspace{-4mm}
		\end{figure}
		Fig. \ref{fig:ImpactofClusterNumber} illustrates the impact of thresholds on clustering. In general, the lower the Davies-Bouldin index and the higher the Silhouette score, the better the threshold will be. For K-medoids as depicted in Fig. \ref{fig:K_Medoid_Cluster_Score}, we select the range [15,30] where the number of clusters is higher and the indicators highlight a good clustering quality. Subsequently, from the selected range we choose the values, 20, 22, 24, and 26 for K, where they return a higher number of high-quality clusters. Similarly, Fig. \ref{fig:EPS_Clusters} and \ref{fig:Siloh_Davis_Eps} study the impact of $\epsilon$ on the number and the quality of clusters in the DBSCAN method. Fig. \ref{fig:EPS_Clusters} shows the number of clusters where $\epsilon$ varies. Here the threshold range of [0.325,0.475] supports the highest number of clusters, which is more than 15. Consequently, Fig. \ref{fig:Siloh_Davis_Eps} can analyze the clustering scores against various $\epsilon$ values, where we aim to find the thresholds in the selected range. We notice that when the value of $\epsilon$ grows bigger than 0.4, both the number of concepts and the quality metrics reduce. We then nominate 0.36,0.38,0.4, and 0.42 for $\epsilon$ to maximize the number of high standard clusters.
		\vspace{-3mm}
		\subsubsection{Selection of Clustering Thresholds}
		\label{Impact of zeta on selected}
		\vspace{-1mm}
		\begin{table*}[t]
			\centering
			\caption{Precision of user concept vectors}
			\vspace{-3mm}
			\def\arraystretch{1.5}
			\tiny
			\begin{tabular}{cc|c|c|c|c|}
				\cline{3-6}
				&                                                                                                          & \multicolumn{4}{c|}{\textbf{Tweet vector combination type}}                                      \\ \hline
				\multicolumn{1}{|c|}{\multirow{2}{*}{\textbf{\begin{tabular}[c]{@{}c@{}}Embedding\end{tabular}}}} & \multirow{2}{*}{\textbf{\begin{tabular}[c]{@{}c@{}}Author concept vector\\cluster type\end{tabular}}} & \multicolumn{2}{c|}{\textbf{Avg}} & \multicolumn{2}{c|}{\textbf{Sum}} \\ \cline{3-6} 
				\multicolumn{1}{|c|}{}                                                                                                 &                                                                                                          & \textbf{$\mathcal{P}_{Textual}$}        & \textbf{$\mathcal{P}_{Conceptual}$}        & \textbf{$\mathcal{P}_{Textual}$}       & \textbf{$\mathcal{P}_{Conceptual}$}       \\ \hline
				\multicolumn{1}{|c|}{\multirow{2}{*}{CBOW}}                                                                            & K-Medoids ($K=22$)                                                                                                & 0.10869                & 0.08405                & 0.10869                & 0.08405                \\ \cline{2-6} 
				\multicolumn{1}{|c|}{}                                                                                                 & DBScan  ($\epsilon=0.36$)                                                                                                 & 0.09166                & 0.06666                & 0.09166               & 0.06666                \\ \hline
				\multicolumn{1}{|c|}{\multirow{2}{*}{Collective}}                                                                           & K-Medoids ($K=22$)                                                                                                & \textbf{0.17291}                 & \textbf{0.12638}                 & \textbf{0.17291}                & \textbf{0.12638}                \\ \cline{2-6} 
				\multicolumn{1}{|c|}{}                                                                                                 & DBScan ($\epsilon=0.36$)                                                                                                  & 0.16388                & 0.12592                & 0.16388                & 0.12592                \\ \hline
			\end{tabular}
			\label{tab:authorconcept}
			\vspace{-6mm}
		\end{table*}
		We limited the range of thresholds in Section \ref{Impact of short-text vector clustering}. In this section, we choose the best final clustering thresholds that are voted by human experts.\\ 
		\textbf{Benchmark -} To select the best thresholds in K-medoid and DBSCAN, we consider the clusters that have been retrieved by each threshold. We then choose the top 10 most similar pair of tweets from each tweet cluster, where the similarity measure is carried out by the well-known TF-IDF method. Subsequently, the human experts of five determine the similarity of the pairs through majority voting using $\rho_0$,$\rho_1$,$\rho_2$, and $\rho_3$ that are then used to compute $\mathcal{P}_{Textual}$ (Section \ref{Impact of embedding and aggregation approaches on content vectors}). The best threshold should generate the clusters with the highest weighted precision.
		\vspace{-4mm}
		\begin{figure}[H]	
			\centering
			\begin{subfigure}{0.49\linewidth}
				\centering
				\includegraphics[width=1\textwidth]{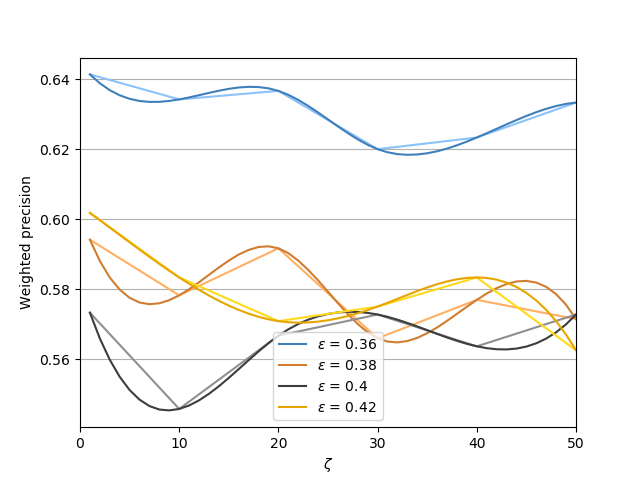}
				\caption{DBSCAN}
				\label{fig:DB-Scan}
			\end{subfigure}
			\centering
			\begin{subfigure}{0.49\linewidth}
				\centering
				\includegraphics[width=1\textwidth]{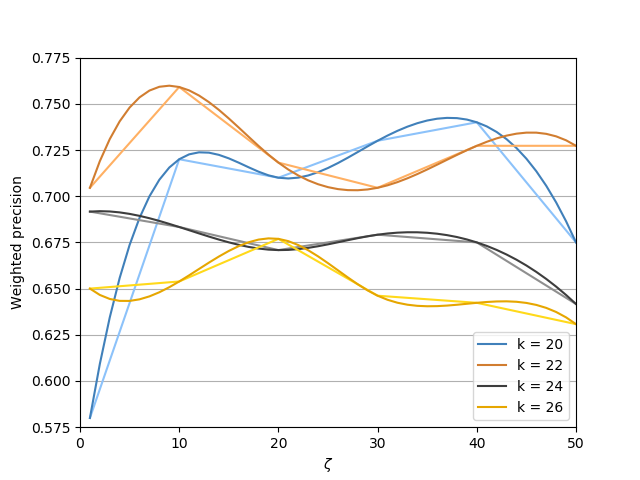}
				\caption{K-Medoids}
				\label{fig:kmeans-k}
			\end{subfigure}
			\hfill
			\vspace{-2mm}
			\caption{The weighted precision by $\zeta$ for various thresholds}
			\label{fig:ClusteringThresholds}
			\vspace{-2mm}
		\end{figure}
		\noindent Fig. \ref{fig:ClusteringThresholds} depicts the weighted precision based on the selected thresholds when the $\zeta$ (Section \ref{Sec:Word_Embedding}) varies.
		For DBSCAN (Fig \ref{fig:DB-Scan}), the value of 0.36 for $\epsilon$ can demonstrate the best performance for all $\zeta$ values. However, other thresholds turn up untrusted with many perturbations where $\zeta$ varies. As Fig. \ref{fig:kmeans-k} illustrates none of the thresholds performs significantly better than the others. Nevertheless, for $K=22$, the k-medoids model gains the highest quality at $\zeta=10$ and maintains the approximate precision for different values of $\zeta$.
		\vspace{-3mm}
		%
		\subsubsection{Impact of Clustering on Authors Concept Vectors}
		\vspace{-1mm}
		As elucidated in Section \ref{Impact of embedding and aggregation approaches on content vectors}, in this section we evaluate the precision of author concept vectors via two weighted metrics of $P_{Conceptual}$ and $P_{Textual}$. As shown in Table \ref{tab:authorconcept}, we report the weighted precisions based on three variations: (1) embedding type (CBOW vs. Collective), (2) the type of combination (Avg vs. Sum) for word vectors in generation of tweet vectors, (3) clustering type (K-Medoids vs DBSCAN) in constructing author concept vectors.
		As shown in Table \ref{tab:authorconcept}, our proposed time-aware collective model can outperform the CBOW model in both weighted precisions, where the overall improvement for $P_{Texutal}$ and $P_{Conceptual}$ are approximately 7 and 4 percent. The K-Medoids clustering performs better than DBSCAN. This is because the DBSCAN model can ignore outliers. We notice that the time-aware collective model performs the best @K=22, where the CBOW gains the lowest results. Since the normalized summation vectors resemble the average approach, their corresponding precision results turn the same. We can, therefore, overlook the impact of combination type in tweet generation.
		\vspace{-3mm}
		\subsubsection{Effect of vectors on Author Subgraph Mining}
		\label{Effect-of-vectors-on-Author-Subgraph-Mining}
		\vspace{-1mm}
		Author similarity matrices (denoted by $X^{Concept}$ and $X^{Content}$), can be combined by $\alpha$ to form the contextual author similarity matrix, denoted by $X^{Total-\alpha}$ (Section \ref{Generating author vectors}). We study the impact of $\alpha$ on the effectiveness of our approach that is measured by $\mathcal{P}_{Textual}$ and $\mathcal{P}_{Conceptual}$.\\
		\vspace{-8mm}
		\begin{figure}[H]
			\centering
			\includegraphics[width=0.6\linewidth]{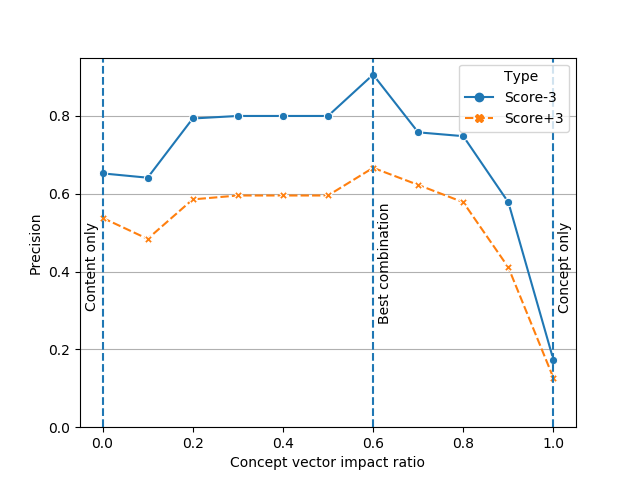}
			\vspace{-2mm}
			\caption{Impact of $\alpha$ (Concept impact ratio) on effectiveness}
			\label{fig:ConceptVectorImpact}
			\vspace{-4mm}
		\end{figure}
		As shown in Figure \ref{fig:ConceptVectorImpact}, $X^{Total-\alpha}$ provides the best precision in both metrics when $\alpha$ is set to 0.6. It is found that the effectiveness of author subgraph mining stops growing at $\alpha=0.6$. We notice that the decrease in performance becomes faster when $\alpha$ increases over 0.8. This can be explained by two rationales: First, the number of exploited concepts are limited to the current dataset, Second, the importance of the embedding process, reflected by $X^{Content}$, cannot be sacrificed in favor of the concept matrix.
		\vspace{-5mm}	
		\section{Conclusion}
		\label{conclusion}
		\vspace{-1mm}
		In this paper, we devise a novel framework that consumes short-text contents (e.g. tweets) to exploit subgraphs including highly correlated authors. To this end, we first need to link authors through computing the similarity edge weights between them, which results in the authors' weighted graph. Primarily, the time-aware word embedding model considers temporal-textual evidence to infer the similarity between temporal splits in multiple dimensions (e.g. Monday and Tuesday in day dimension) and collectively computes the word vector representations.  Subsequently, we obtain short-text vectors and \textit{author content vectors} using primary word vectors. Similarly, \textit{author concept vectors} represent how every author is relevant to each of the short-text clusters. We then fuse the content-based and conceptual author similarities to calculate the correlation weight between each pair of authors. Consequently, given the authors' weighted graph, the stack-wise graph cutting component in our framework can extract the maximum spanning trees that establish the subgraph of linked authors. The result of the extensive experiments on a real-world microblog dataset proves the superiority of our proposed model in short-text author linking. Moreover, we notice that compared to DBSCAN, the k-medoids clustering can better discover the concepts from tweet contents.\\
		Naturally, the short-texts differ in significance (e.g. popularity). Hence, to nominate the concepts from short-text clusters, we should not only consider the relevance of the short-texts but also grant higher importance to the concepts of those with higher popularity. We leave this task for future work.
		\vspace{-2mm}
		\ifCLASSOPTIONcaptionsoff
		\newpage
		\fi
		
		
		
		\vspace{-5mm}
		\bibliographystyle{IEEEtran}
		\bibliography{IEEEexample}

\begin{thebibliography}{10}
\providecommand{\url}[1]{#1}
\csname url@samestyle\endcsname
\providecommand{\newblock}{\relax}
\providecommand{\bibinfo}[2]{#2}
\providecommand{\BIBentrySTDinterwordspacing}{\spaceskip=0pt\relax}
\providecommand{\BIBentryALTinterwordstretchfactor}{4}
\providecommand{\BIBentryALTinterwordspacing}{\spaceskip=\fontdimen2\font plus
\BIBentryALTinterwordstretchfactor\fontdimen3\font minus
  \fontdimen4\font\relax}
\providecommand{\BIBforeignlanguage}[2]{{%
\expandafter\ifx\csname l@#1\endcsname\relax
\typeout{** WARNING: IEEEtran.bst: No hyphenation pattern has been}%
\typeout{** loaded for the language `#1'. Using the pattern for}%
\typeout{** the default language instead.}%
\else
\language=\csname l@#1\endcsname
\fi
#2}}
\providecommand{\BIBdecl}{\relax}
\BIBdecl

\bibitem{chen2018attention}
J.~Chen, F.~Zhuang, X.~Hong, X.~Ao, X.~Xie, and Q.~He, ``Attention-driven
  factor model for explainable personalized recommendation,'' in \emph{The 41st
  International ACM SIGIR Conference on Research \& Development in Information
  Retrieval}.\hskip 1em plus 0.5em minus 0.4em\relax ACM, 2018, pp. 909--912.

\bibitem{livne2014citesight}
A.~Livne, V.~Gokuladas, J.~Teevan, S.~T. Dumais, and E.~Adar, ``Citesight:
  supporting contextual citation recommendation using differential search,'' in
  \emph{Proceedings of the 37th international ACM SIGIR conference on Research
  \& development in information retrieval}.\hskip 1em plus 0.5em minus
  0.4em\relax ACM, 2014, pp. 807--816.

\bibitem{cao2017embedding}
D.~Cao, L.~Nie, X.~He, X.~Wei, S.~Zhu, and T.-S. Chua, ``Embedding
  factorization models for jointly recommending items and user generated
  lists,'' in \emph{Proceedings of the 40th International ACM SIGIR Conference
  on Research and Development in Information Retrieval}.\hskip 1em plus 0.5em
  minus 0.4em\relax ACM, 2017, pp. 585--594.

\bibitem{manotumruksa2018contextual}
J.~Manotumruksa, C.~Macdonald, and I.~Ounis, ``A contextual attention recurrent
  architecture for context-aware venue recommendation,'' in \emph{The 41st
  International ACM SIGIR Conference on Research \& Development in Information
  Retrieval}.\hskip 1em plus 0.5em minus 0.4em\relax ACM, 2018, pp. 555--564.

\bibitem{cao2018attentive}
D.~Cao, X.~He, L.~Miao, Y.~An, C.~Yang, and R.~Hong, ``Attentive group
  recommendation,'' in \emph{The 41st International ACM SIGIR Conference on
  Research \& Development in Information Retrieval}.\hskip 1em plus 0.5em minus
  0.4em\relax ACM, 2018, pp. 645--654.

\bibitem{cai2017community}
H.~Cai, V.~W. Zheng, F.~Zhu, K.~C.-C. Chang, and Z.~Huang, ``From community
  detection to community profiling,'' \emph{Proceedings of the VLDB Endowment},
  vol.~10, no.~7, pp. 817--828, 2017.

\bibitem{belesiotis2018spatio}
A.~Belesiotis, D.~Skoutas, C.~Efstathiades, V.~Kaffes, and D.~Pfoser,
  ``Spatio-textual user matching and clustering based on set similarity
  joins,'' \emph{The VLDB Journal—The International Journal on Very Large
  Data Bases}, vol.~27, no.~3, pp. 297--320, 2018.

\bibitem{hosseini2018mining}
S.~Hosseini, H.~Yin, M.~Zhang, Y.~Elovici, and X.~Zhou, ``Mining subgraphs from
  propagation networks through temporal dynamic analysis,'' in \emph{2018 19th
  IEEE International Conference on Mobile Data Management (MDM)}.\hskip 1em
  plus 0.5em minus 0.4em\relax IEEE, 2018, pp. 66--75.

\bibitem{hosseini2018exploiting}
S.~Hosseini, H.~Yin, N.-M. Cheung, K.~P. Leng, Y.~Elovici, and X.~Zhou,
  ``Exploiting reshaping subgraphs from bilateral propagation graphs,'' in
  \emph{International Conference on Database Systems for Advanced
  Applications}.\hskip 1em plus 0.5em minus 0.4em\relax Springer, 2018, pp.
  342--351.

\bibitem{Zhang2016}
Y.~Zhang, A.~Adiga, S.~Saha, A.~Vullikanti, and B.~A. Prakash, ``Near-optimal
  algorithms for controlling propagation at group scale on networks,''
  \emph{IEEE Transactions on Knowledge and Data Engineering}, vol.~28, no.~12,
  pp. 3339--3352, 2016.

\bibitem{ganguly2016author2vec}
S.~Ganguly, M.~Gupta, V.~Varma, V.~Pudi \emph{et~al.}, ``Author2vec: Learning
  author representations by combining content and link information,'' in
  \emph{Proceedings of the 25th International Conference Companion on World
  Wide Web}.\hskip 1em plus 0.5em minus 0.4em\relax International World Wide
  Web Conferences Steering Committee, 2016, pp. 49--50.

\bibitem{Rosen-Zvi2012}
M.~Rosen-Zvi, T.~Griffiths, M.~Steyvers, and P.~Smyth, ``The author-topic model
  for authors and documents,'' \emph{in Proceedings of the 20th Conference on
  Uncertainty in Artificial Intelligence}, 2004.

\bibitem{nguyen2015improving}
D.~Q. Nguyen, R.~Billingsley, L.~Du, and M.~Johnson, ``Improving topic models
  with latent feature word representations,'' \emph{Transactions of the
  Association for Computational Linguistics}, vol.~3, pp. 299--313, 2015.

\bibitem{cao2017you}
C.~Cao, H.~Ge, H.~Lu, X.~Hu, and J.~Caverlee, ``What are you known for?:
  Learning user topical profiles with implicit and explicit footprints,'' in
  \emph{Proceedings of the 40th International ACM SIGIR Conference on Research
  and Development in Information Retrieval}.\hskip 1em plus 0.5em minus
  0.4em\relax ACM, 2017, pp. 743--752.

\bibitem{Hosseini2014}
S.~Hosseini, S.~Unankard, X.~Zhou, and S.~Sadiq, ``Location oriented phrase
  detection in microblogs,'' in \emph{International Conference on Database
  Systems for Advanced Applications}.\hskip 1em plus 0.5em minus 0.4em\relax
  Springer, 2014, pp. 495--509.

\bibitem{hosseini2017leveraging}
S.~Hosseini, H.~Yin, X.~Zhou, S.~Sadiq, M.~R. Kangavari, and N.-M. Cheung,
  ``Leveraging multi-aspect time-related influence in location
  recommendation,'' \emph{World Wide Web}, vol.~22, no.~3, pp. 1001--1028,
  2019.

\bibitem{Pennington2014}
J.~Pennington, R.~Socher, and C.~Manning, ``Glove: Global vectors for word
  representation,'' in \emph{Proceedings of the 2014 Conference on Empirical
  Methods in Natural Language Processing ({EMNLP})}.\hskip 1em plus 0.5em minus
  0.4em\relax Association for Computational Linguistics, 2014.

\bibitem{mikolov2013efficient}
T.~Mikolov, K.~Chen, G.~Corrado, and J.~Dean, ``Efficient estimation of word
  representations in vector space,'' \emph{arXiv preprint arXiv:1301.3781},
  2013.

\bibitem{dumais2004latent}
S.~T. Dumais, ``Latent semantic analysis,'' \emph{Annual review of information
  science and technology}, vol.~38, no.~1, pp. 188--230, 2004.

\bibitem{Mikolov2013a}
T.~Mikolov, I.~Sutskever, K.~Chen, G.~Corrado, and J.~Dean, ``Distributed
  representations of words and phrases and their compositionality,''
  \emph{Advances in Neural Information Processing Systems 26 (NIPS 2013)},
  2013.

\bibitem{alinani2018aggregating}
K.~Alinani, A.~Alinani, D.~H. Narejo, and G.~Wang, ``Aggregating author
  profiles from multiple publisher networks to build a list of potential
  collaborators,'' \emph{IEEE Access}, vol.~6, pp. 20\,298--20\,308, 2018.

\bibitem{alshareef2018recommending}
A.~M. Alshareef, M.~F. Alhamid, and A.~El~Saddik, ``Recommending scientific
  collaboration based on topical, authors and venues similarities,'' in
  \emph{2018 IEEE International Conference on Information Reuse and Integration
  (IRI)}.\hskip 1em plus 0.5em minus 0.4em\relax IEEE, 2018, pp. 55--61.

\bibitem{li2018conference}
S.~Li, P.~Brusilovsky, S.~Su, and X.~Cheng, ``Conference paper recommendation
  for academic conferences,'' \emph{IEEE Access}, vol.~6, pp. 17\,153--17\,164,
  2018.

\bibitem{li2019personalised}
X.~Li, Y.~Chen, B.~Pettit, and M.~D. Rijke, ``Personalised reranking of paper
  recommendations using paper content and user behavior,'' \emph{ACM
  Transactions on Information Systems (TOIS)}, vol.~37, no.~3, p.~31, 2019.

\bibitem{Hua2017}
W.~Hua, Z.~Wang, H.~Wang, K.~Zheng, and X.~Zhou, ``Understand short texts by
  harvesting and analyzing semantic knowledge,'' \emph{{IEEE} Transactions on
  Knowledge and Data Engineering}, vol.~29, no.~3, pp. 499--512, mar 2017.

\bibitem{rumelhart1986learning}
D.~E. Rumelhart, G.~E. Hinton, and R.~J. Williams, ``Learning representations
  by back-propagating errors,'' \emph{nature}, vol. 323, no. 6088, p. 533,
  1986.

\bibitem{ling2015two}
W.~Ling, C.~Dyer, A.~W. Black, and I.~Trancoso, ``Two/too simple adaptations of
  word2vec for syntax problems,'' in \emph{Proceedings of the 2015 Conference
  of the North American Chapter of the Association for Computational
  Linguistics: Human Language Technologies}, 2015, pp. 1299--1304.

\bibitem{hu2017opinion}
Y.-H. Hu, Y.-L. Chen, and H.-L. Chou, ``Opinion mining from online hotel
  reviews--a text summarization approach,'' \emph{Information Processing And
  Management}, vol.~53, no.~2, pp. 436--449, 2017.

\bibitem{majumder2017deep}
N.~Majumder, S.~Poria, A.~Gelbukh, and E.~Cambria, ``Deep learning-based
  document modeling for personality detection from text,'' \emph{IEEE
  Intelligent Systems}, vol.~32, no.~2, pp. 74--79, 2017.

\bibitem{park2018conceptvector}
D.~Park, S.~Kim, J.~Lee, J.~Choo, N.~Diakopoulos, and N.~Elmqvist,
  ``Conceptvector: text visual analytics via interactive lexicon building using
  word embedding,'' \emph{IEEE transactions on visualization and computer
  graphics}, vol.~24, no.~1, pp. 361--370, 2018.

\bibitem{grbovic2015commerce}
M.~Grbovic, V.~Radosavljevic, N.~Djuric, N.~Bhamidipati, J.~Savla, V.~Bhagwan,
  and D.~Sharp, ``E-commerce in your inbox: Product recommendations at scale,''
  in \emph{Proceedings of the 21th ACM SIGKDD International Conference on
  Knowledge Discovery and Data Mining}.\hskip 1em plus 0.5em minus 0.4em\relax
  ACM, 2015, pp. 1809--1818.

\bibitem{fu2018novel}
M.~Fu, H.~Qu, Z.~Yi, L.~Lu, and Y.~Liu, ``A novel deep learning-based
  collaborative filtering model for recommendation system,'' \emph{IEEE
  Transactions on Cybernetics}, 2018.

\bibitem{Deerwester1990}
S.~Deerwester, S.~T. Dumais, G.~W. Furnas, T.~K. Landauer, and R.~Harshman,
  ``Indexing by latent semantic analysis,'' \emph{Journal of the American
  Society for Information Science}, vol.~41, no.~6, pp. 391--407, sep 1990.

\bibitem{Manning2008}
C.~D. Manning, P.~Raghavan, and H.~Schütze, \emph{Introduction to Information
  Retrieval}.\hskip 1em plus 0.5em minus 0.4em\relax Cambridge University
  Press, 2008.

\bibitem{ling2015not}
W.~Ling, Y.~Tsvetkov, S.~Amir, R.~Fermandez, C.~Dyer, A.~W. Black, I.~Trancoso,
  and C.-C. Lin, ``Not all contexts are created equal: Better word
  representations with variable attention,'' in \emph{Proceedings of the 2015
  Conference on Empirical Methods in Natural Language Processing}, 2015, pp.
  1367--1372.

\bibitem{talley2011database}
E.~M. Talley, D.~Newman, D.~Mimno, B.~W. Herr~II, H.~M. Wallach, G.~A. Burns,
  A.~M. Leenders, and A.~McCallum, ``Database of nih grants using
  machine-learned categories and graphical clustering,'' \emph{Nature Methods},
  vol.~8, no.~6, p. 443, 2011.

\bibitem{le2014distributed}
Q.~Le and T.~Mikolov, ``Distributed representations of sentences and
  documents,'' in \emph{International Conference on Machine Learning}, 2014,
  pp. 1188--1196.

\bibitem{zhu2018exploiting}
G.~Zhu and C.~A. Iglesias, ``Exploiting semantic similarity for named entity
  disambiguation in knowledge graphs,'' \emph{Expert Systems with
  Applications}, vol. 101, pp. 8--24, 2018.

\bibitem{bamler2017dynamic}
R.~Bamler and S.~Mandt, ``Dynamic word embeddings via skip-gram filtering,''
  \emph{stat}, vol. 1050, p.~27, 2017.

\bibitem{dubossarsky2017outta}
H.~Dubossarsky, D.~Weinshall, and E.~Grossman, ``Outta control: Laws of
  semantic change and inherent biases in word representation models,'' in
  \emph{Proceedings of the 2017 conference on empirical methods in natural
  language processing}, 2017, pp. 1136--1145.

\bibitem{rosin2017learning}
G.~D. Rosin, E.~Adar, and K.~Radinsky, ``Learning word relatedness over time,''
  \emph{arXiv preprint arXiv:1707.08081}, 2017.

\bibitem{mikolov2013linguistic}
T.~Mikolov, W.-t. Yih, and G.~Zweig, ``Linguistic regularities in continuous
  space word representations,'' in \emph{Proceedings of the 2013 Conference of
  the North American Chapter of the Association for Computational Linguistics:
  Human Language Technologies}, 2013, pp. 746--751.

\bibitem{seyler2018information}
D.~Seyler, P.~Chandar, and M.~Davis, ``An information retrieval framework for
  contextual suggestion based on heterogeneous information network
  embeddings,'' in \emph{The 41st International ACM SIGIR Conference on
  Research \& Development in Information Retrieval}.\hskip 1em plus 0.5em minus
  0.4em\relax ACM, 2018, pp. 953--956.

\bibitem{koren2009matrix}
Y.~Koren, R.~Bell, and C.~Volinsky, ``Matrix factorization techniques for
  recommender systems,'' \emph{Computer}, no.~8, pp. 30--37, 2009.

\bibitem{lee1999learning}
D.~D. Lee and H.~S. Seung, ``Learning the parts of objects by non-negative
  matrix factorization,'' \emph{Nature}, vol. 401, no. 6755, p. 788, 1999.

\bibitem{li2017neural}
P.~Li, Z.~Wang, Z.~Ren, L.~Bing, and W.~Lam, ``Neural rating regression with
  abstractive tips generation for recommendation,'' in \emph{Proceedings of the
  40th International ACM SIGIR conference on Research and Development in
  Information Retrieval}.\hskip 1em plus 0.5em minus 0.4em\relax ACM, 2017, pp.
  345--354.

\bibitem{chen2010short}
J.~Chen, R.~Nairn, L.~Nelson, M.~Bernstein, and E.~Chi, ``Short and tweet:
  experiments on recommending content from information streams,'' in
  \emph{Proceedings of the SIGCHI Conference on Human Factors in Computing
  Systems}.\hskip 1em plus 0.5em minus 0.4em\relax ACM, 2010, pp. 1185--1194.

\bibitem{lian2015content}
D.~Lian, Y.~Ge, F.~Zhang, N.~J. Yuan, X.~Xie, T.~Zhou, and Y.~Rui,
  ``Content-aware collaborative filtering for location recommendation based on
  human mobility data,'' in \emph{2015 IEEE International Conference on Data
  Mining}.\hskip 1em plus 0.5em minus 0.4em\relax IEEE, 2015, pp. 261--270.

\bibitem{grover2016node2vec}
A.~Grover and J.~Leskovec, ``node2vec: Scalable feature learning for
  networks,'' in \emph{Proceedings of the 22nd ACM SIGKDD international
  conference on Knowledge discovery and data mining}.\hskip 1em plus 0.5em
  minus 0.4em\relax ACM, 2016, pp. 855--864.

\bibitem{liu2018multi}
H.~Liu, L.~Wu, D.~Zhang, M.~Jian, and X.~Zhang, ``Multi-perspective user2vec:
  Exploiting re-pin activity for user representation learning in content
  curation social network,'' \emph{Signal Processing}, vol. 142, pp. 450--456,
  2018.

\bibitem{perozzi2014deepwalk}
B.~Perozzi, R.~Al-Rfou, and S.~Skiena, ``Deepwalk: Online learning of social
  representations,'' in \emph{Proceedings of the 20th ACM SIGKDD international
  conference on Knowledge discovery and data mining}.\hskip 1em plus 0.5em
  minus 0.4em\relax ACM, 2014, pp. 701--710.

\bibitem{wu2017mllda}
L.~Wu, D.~Wang, X.~Zhang, S.~Liu, L.~Zhang, and C.~W. Chen, ``Mllda:
  Multi-level lda for modelling users on content curation social networks,''
  \emph{Neurocomputing}, vol. 236, pp. 73--81, 2017.

\bibitem{yan2012tweet}
R.~Yan, M.~Lapata, and X.~Li, ``Tweet recommendation with graph co-ranking,''
  in \emph{Proceedings of the 50th Annual Meeting of the Association for
  Computational Linguistics: Long Papers-Volume 1}.\hskip 1em plus 0.5em minus
  0.4em\relax Association for Computational Linguistics, 2012, pp. 516--525.

\bibitem{alharbi2016learning}
B.~Alharbi and X.~Zhang, ``Learning from your network of friends: a trajectory
  representation learning model based on online social ties,'' in \emph{2016
  IEEE 16th International Conference on Data Mining (ICDM)}.\hskip 1em plus
  0.5em minus 0.4em\relax IEEE, 2016, pp. 781--786.

\bibitem{fani2017temporally}
H.~Fani, E.~Bagheri, and W.~Du, ``Temporally like-minded user community
  identification through neural embeddings,'' in \emph{Proceedings of the 2017
  ACM on Conference on Information and Knowledge Management}.\hskip 1em plus
  0.5em minus 0.4em\relax ACM, 2017, pp. 577--586.

\bibitem{graves2012supervised}
A.~Graves, ``Supervised sequence labelling,'' in \emph{Supervised sequence
  labelling with recurrent neural networks}.\hskip 1em plus 0.5em minus
  0.4em\relax Springer, 2012, pp. 5--13.

\bibitem{shen2014learning}
Y.~Shen, X.~He, J.~Gao, L.~Deng, and G.~Mesnil, ``Learning semantic
  representations using convolutional neural networks for web search,'' in
  \emph{Proceedings of the 23rd International Conference on World Wide
  Web}.\hskip 1em plus 0.5em minus 0.4em\relax ACM, 2014, pp. 373--374.

\bibitem{buckley1995automatic}
C.~Buckley, G.~Salton, J.~Allan, and A.~Singhal, ``Automatic query expansion
  using smart: Trec 3,'' \emph{NIST special publication sp}, pp. 69--69, 1995.

\bibitem{efron2012improving}
M.~Efron, P.~Organisciak, and K.~Fenlon, ``Improving retrieval of short texts
  through document expansion,'' in \emph{Proceedings of the 35th international
  ACM SIGIR conference on Research and development in information
  retrieval}.\hskip 1em plus 0.5em minus 0.4em\relax ACM, 2012, pp. 911--920.

\bibitem{tang2017end}
J.~Tang, Y.~Wang, K.~Zheng, and Q.~Mei, ``End-to-end learning for short text
  expansion,'' in \emph{Proceedings of the 23rd ACM SIGKDD International
  Conference on Knowledge Discovery and Data Mining}.\hskip 1em plus 0.5em
  minus 0.4em\relax ACM, 2017, pp. 1105--1113.

\bibitem{zhai2001model}
C.~Zhai and J.~Lafferty, ``Model-based feedback in the language modeling
  approach to information retrieval,'' in \emph{Proceedings of the tenth
  international conference on Information and knowledge management}.\hskip 1em
  plus 0.5em minus 0.4em\relax ACM, 2001, pp. 403--410.

\bibitem{blei2009topic}
D.~M. Blei and J.~D. Lafferty, ``Topic models,'' in \emph{Text Mining}.\hskip
  1em plus 0.5em minus 0.4em\relax Chapman and Hall/CRC, 2009, pp. 101--124.

\bibitem{li2016topic}
C.~Li, H.~Wang, Z.~Zhang, A.~Sun, and Z.~Ma, ``Topic modeling for short texts
  with auxiliary word embeddings,'' in \emph{Proceedings of the 39th
  International ACM SIGIR conference on Research and Development in Information
  Retrieval}.\hskip 1em plus 0.5em minus 0.4em\relax ACM, 2016, pp. 165--174.

\bibitem{kim2013context}
D.~Kim, H.~Wang, and A.~H. Oh, ``Context-dependent conceptualization.'' in
  \emph{IJCAI}, 2013, pp. 2654--2661.

\bibitem{song2011short}
Y.~Song, H.~Wang, Z.~Wang, H.~Li, and W.~Chen, ``Short text conceptualization
  using a probabilistic knowledgebase,'' in \emph{Proceedings of the
  twenty-second international joint conference on artificial
  intelligence-volume volume three}.\hskip 1em plus 0.5em minus 0.4em\relax
  AAAI Press, 2011, pp. 2330--2336.

\bibitem{diaz2012real}
E.~Diaz-Aviles, L.~Drumond, L.~Schmidt-Thieme, and W.~Nejdl, ``Real-time top-n
  recommendation in social streams,'' in \emph{Proceedings of the sixth ACM
  conference on Recommender systems}.\hskip 1em plus 0.5em minus 0.4em\relax
  ACM, 2012, pp. 59--66.

\bibitem{uysal2011user}
I.~Uysal and W.~B. Croft, ``User oriented tweet ranking: a filtering approach
  to microblogs,'' in \emph{Proceedings of the 20th ACM international
  conference on Information and knowledge management}.\hskip 1em plus 0.5em
  minus 0.4em\relax ACM, 2011, pp. 2261--2264.

\bibitem{McCallum2003}
A.~McCallum and W.~Li, ``Early results for named entity recognition with
  conditional random fields, feature induction and web-enhanced lexicons,'' in
  \emph{Proceedings of the seventh conference on Natural language learning at
  {HLT}-{NAACL} 2003 -}.\hskip 1em plus 0.5em minus 0.4em\relax Association for
  Computational Linguistics, 2003.

\bibitem{Zhou2001}
G.~Zhou and J.~Su, ``Named entity recognition using an {HMM}-based chunk
  tagger,'' in \emph{Proceedings of the 40th Annual Meeting on Association for
  Computational Linguistics {ACL}}.\hskip 1em plus 0.5em minus 0.4em\relax
  Association for Computational Linguistics, 2001.

\bibitem{sugawara2015context}
H.~Sugawara, H.~Takamura, R.~Sasano, and M.~Okumura, ``Context representation
  with word embeddings for wsd,'' in \emph{Conference of the Pacific
  Association for Computational Linguistics}.\hskip 1em plus 0.5em minus
  0.4em\relax Springer, 2015, pp. 108--119.

\bibitem{schubert2017dbscan}
E.~Schubert, J.~Sander, M.~Ester, H.~P. Kriegel, and X.~Xu, ``Dbscan revisited,
  revisited: why and how you should (still) use dbscan,'' \emph{ACM
  Transactions on Database Systems (TODS)}, vol.~42, no.~3, p.~19, 2017.

\bibitem{cheng2015contextual}
J.~Cheng, Z.~Wang, J.-R. Wen, J.~Yan, and Z.~Chen, ``Contextual text
  understanding in distributional semantic space,'' in \emph{Proceedings of the
  24th ACM International on Conference on Information and Knowledge
  Management}.\hskip 1em plus 0.5em minus 0.4em\relax ACM, 2015, pp. 133--142.

\bibitem{hosseini2019teals}
S.~Hosseini, S.~Najafi~Pour, N.-M. Cheung, M.~R. Kangavari, and X.~Zhou,
  ``Teags: Time-aware text embedding approach to generate subgraphs,''
  \emph{arXiv preprint arXiv:1907.03191}, 2019.

\bibitem{Kuzi:2017:QEE:3077136.3080660}
\BIBentryALTinterwordspacing
S.~Kuzi, D.~Carmel, A.~Libov, and A.~Raviv, ``Query expansion for email
  search,'' in \emph{Proceedings of the 40th International ACM SIGIR Conference
  on Research and Development in Information Retrieval}, ser. SIGIR '17.\hskip
  1em plus 0.5em minus 0.4em\relax New York, NY, USA: ACM, 2017, pp. 849--852.
  [Online]. Available: \url{http://doi.acm.org/10.1145/3077136.3080660}
\BIBentrySTDinterwordspacing

\bibitem{beel2016paper}
J.~Beel, B.~Gipp, S.~Langer, and C.~Breitinger, ``paper recommender systems: a
  literature survey,'' \emph{International Journal on Digital Libraries},
  vol.~17, no.~4, pp. 305--338, 2016.

\bibitem{Rosa2011}
K.~D. Rosa, R.~Shah, B.~Lin, A.~Gershman, and R.~Frederking, ``Topical
  clustering of tweets,'' \emph{Proceedings of the ACM SIGIR: SWSM}, vol.~63,
  2011.

\bibitem{rousseeuw1987silhouettes}
P.~J. Rousseeuw, ``Silhouettes: a graphical aid to the interpretation and
  validation of cluster analysis,'' \emph{Journal of computational and applied
  mathematics}, vol.~20, pp. 53--65, 1987.

\bibitem{davies1979cluster}
D.~L. Davies and D.~W. Bouldin, ``A cluster separation measure,'' \emph{IEEE
  transactions on pattern analysis and machine intelligence}, no.~2, pp.
  224--227, 1979.

\end{thebibliography}
		\vspace{-14mm}

\begin{IEEEbiography}[{\includegraphics[width=1.0in,height=1.25in,clip,keepaspectratio]{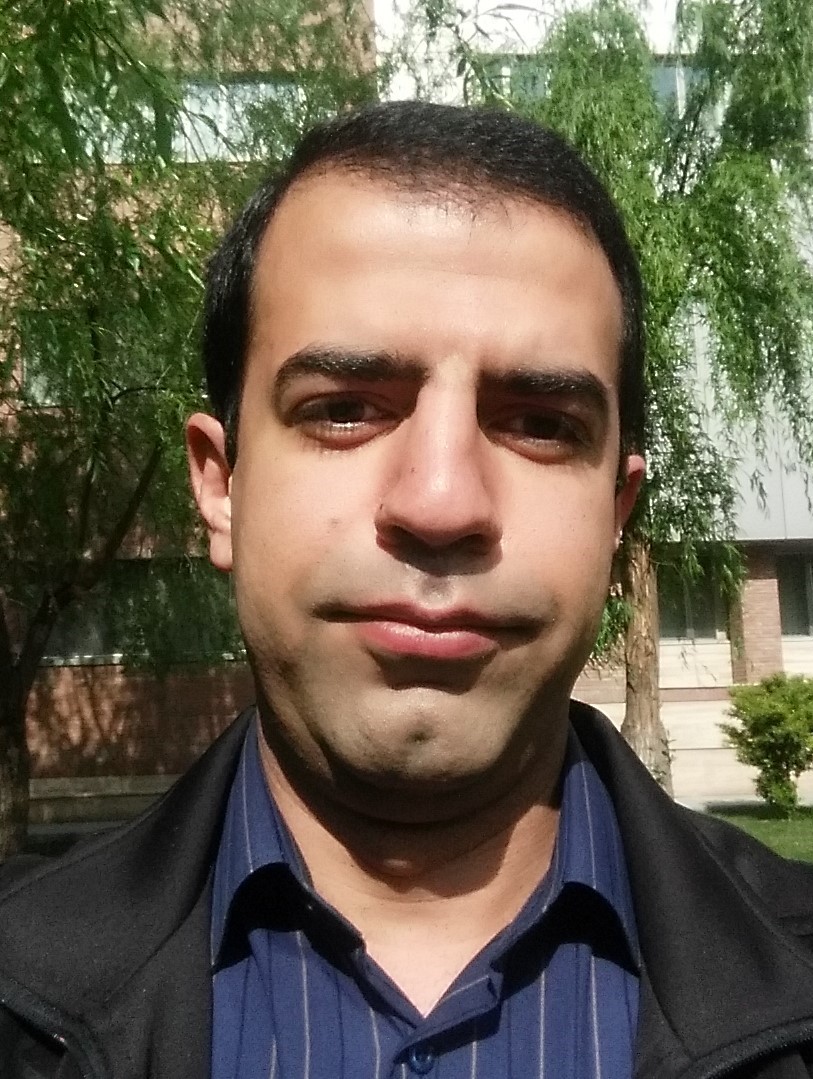}}]{Saeed NajafiPour}
	is a researcher in the computational cognitive models laboratory at Iran University of Science and Technology (IUST). He is to receive his M.Sc. in Software Engineering from IUST and completed his B.S. degree in computer software engineering from the Bahonar University, Iran. His research interests include Natural Language Processing, data mining, deep learning, and trajectory analytics.
\end{IEEEbiography}
\vspace{-14mm}
\begin{IEEEbiography}[{\includegraphics[width=1.0in,height=1.25in,clip,keepaspectratio]{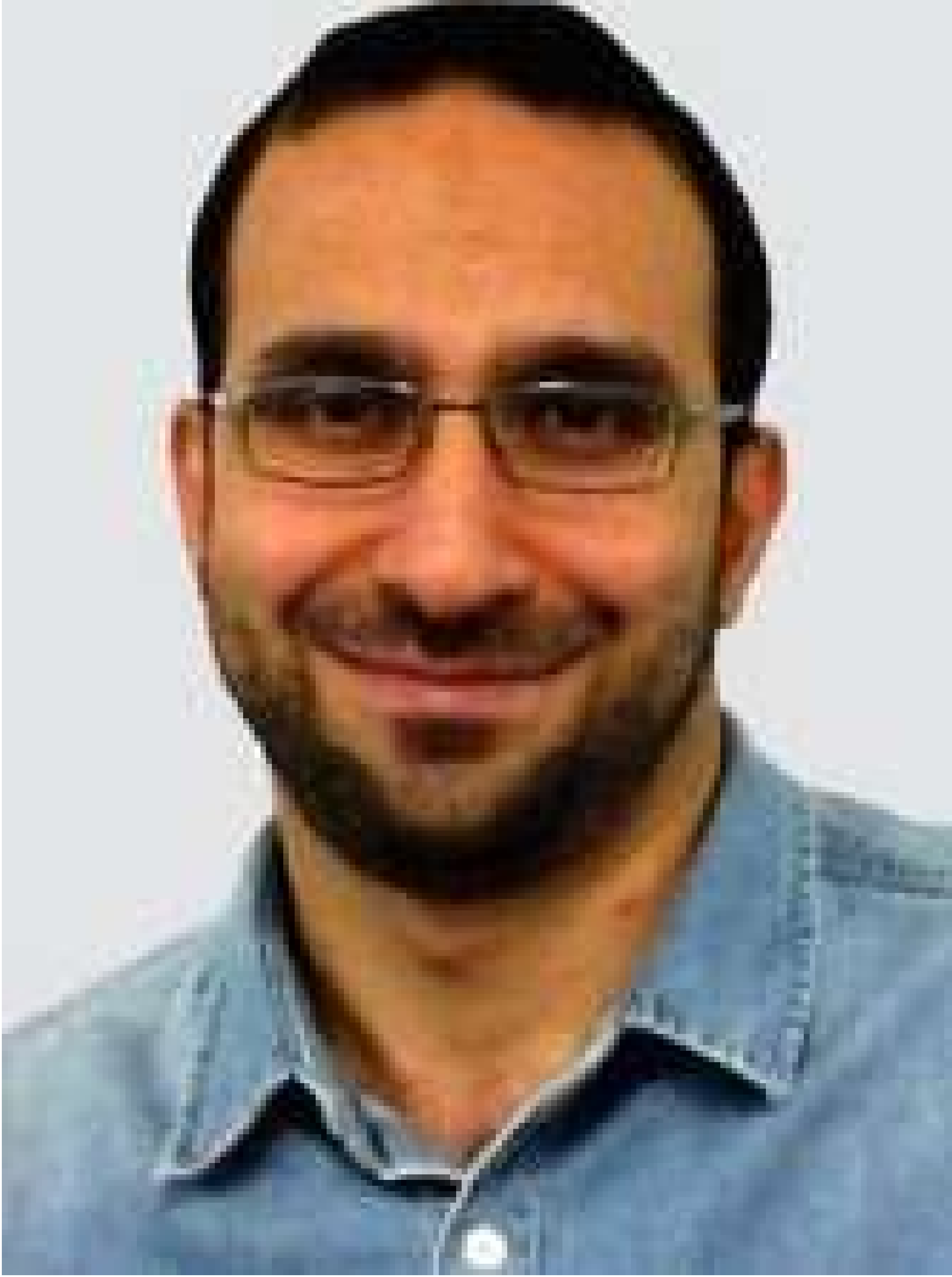}}]{Saeid Hosseini}
	won the Australian Postgraduate Award and received Ph.D. degree in Computer Science from the University of Queensland, Australia, in 2017. As a post-doc research scientist, his research interests include spatiotemporal database, dynamical processes, data and graph mining, big data analytics, recommendation systems, and machine learning. He has been a PC member in CSS and DASFAA and a reviewer in ICDM and TKDE.
\end{IEEEbiography}
\vspace{-14mm}
\begin{IEEEbiography}[{\includegraphics[width=1.0in,height=1.25in,clip,keepaspectratio]{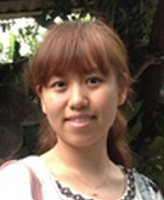}}]{Wen Hua}
currently works as a Lecturer at the University of Queensland, Australia where she was appointed as a Postdoctoral Research Fellow. She received her doctoral and bachelor degrees in Computer Science from Renmin University of China. Her current research interests include natural language processing, information extraction and retrieval, text mining, social media analysis, and spatiotemporal data analytics. She has published articles in reputed journals and conferences including SIGMOD, TKDE, VLDBJ.
\end{IEEEbiography}
\vspace{-14mm}
\begin{IEEEbiography}[{\includegraphics[width=1.0in,height=1.25in,clip,keepaspectratio]{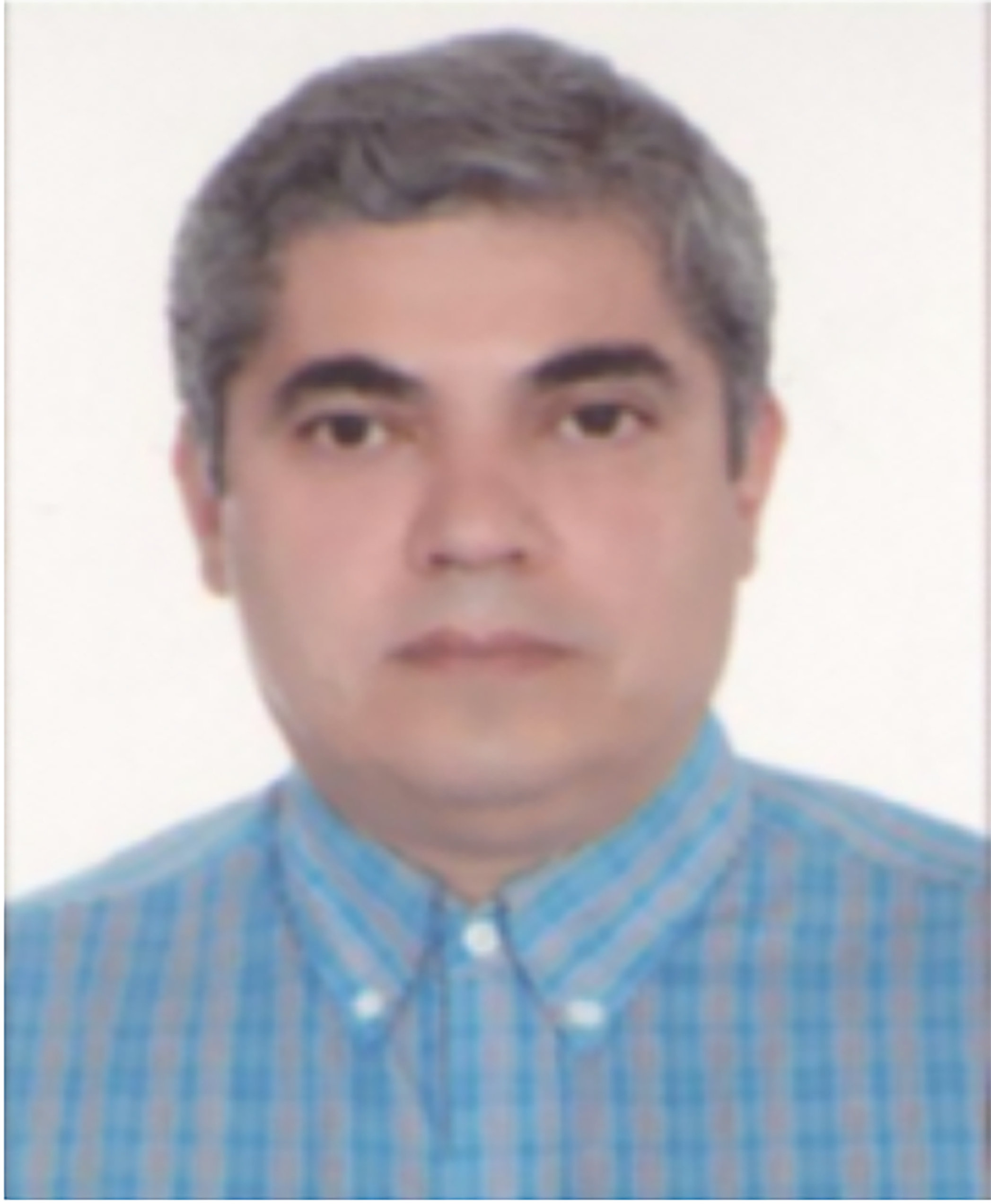}}]{Mohammad Reza Kangavari}
	received B.Sc. in mathematics and computer science from the Sharif University of Technology, M.Sc. in computer science from Salford University, and the Ph.D. in computer science from the University of Manchester in 1994. He is a lecturer in the Computer Engineering Department, Iran University of Science and Technology. His research interests include Intelligent Systems, Human-Computer-Interaction, Cognitive Computing, Data Science, Machine Learning, and Sensor Networks.
\end{IEEEbiography}
\vspace{-14mm}
\begin{IEEEbiography}[{\includegraphics[width=1.0in,height=1.25in,clip,keepaspectratio]{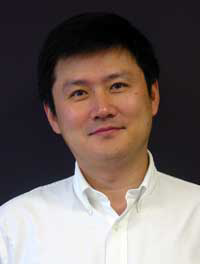}}]{Xiaofang Zhou}
	is a Professor of Computer Science at The University of Queensland, Leader of UQ Data Science Research Group and an IEEE Fellow. He received B.Sc. and M.Sc. in Computer Science from Nanjing University, China, and Ph.D. from UQ. His research interests include spatiotemporal and multimedia databases, data mining, query processing, big data analytics and machine learning, with over 300 publications in SIGMOD, VLDB, ICDE, AAAI, IJCAI, and IEEE and ACM Transactions.
\end{IEEEbiography}






\end{document}